\documentclass[lettersize,journal]{IEEEtran}
\usepackage{amsmath,amsfonts}
\usepackage{array}
\usepackage[caption=false,font=normalsize,labelfont=sf,textfont=sf]{subfig}
\usepackage{textcomp}
\usepackage{stfloats}
\usepackage{url}
\usepackage{verbatim}
\usepackage{graphicx}
\usepackage{cite}
\usepackage[english]{babel}
\usepackage[utf8]{inputenc}
\usepackage{balance}
\usepackage{xspace}
\usepackage{bbm}
\usepackage{rotating}
\usepackage{url}
\usepackage{csquotes}
\usepackage{subfig}
\usepackage{paralist}
\usepackage{multirow}
\usepackage{multicol}
\usepackage{amsmath,amssymb,mathtools,amsthm}
\usepackage{booktabs}
\usepackage{xfrac}
\usepackage{array}
\usepackage{algorithm}
\usepackage[noend]{algpseudocode}
\usepackage{textcomp}
\usepackage{siunitx}
\usepackage{microtype}
\usepackage[usenames,dvipsnames]{xcolor}
\usepackage{pgfplots}
\pgfplotsset{compat=newest,unit code/.code={\si{#1}},plot coordinates/math parser=false,grid style={lightgray}, ylabel right/.style={
        after end axis/.append code={
            \node [rotate=90, anchor=north] at (rel axis cs:1,0.5) {#1};
        }   
    }}
\usepgfplotslibrary{units,external,groupplots,fillbetween}
\usetikzlibrary{positioning,angles,quotes,patterns,shapes,spy, shapes.misc,backgrounds}
\tikzstyle{block} = [draw, rectangle, minimum height=2em, minimum width=5em]
\tikzstyle{addon} = [draw, rectangle, rounded corners]
\tikzstyle{pinstyle} = [pin edge={<-,thin,black}]
\tikzstyle{pinstyle2} = [pin edge={->,thin,black}]
\tikzstyle{mult} = [draw, isosceles triangle]
\tikzstyle{circ} = [draw, circle]
\tikzstyle{coord} = [coordinate]
\tikzstyle{circ2} = [draw, circle,minimum width=3pt, inner sep=0]
\tikzset{>=latex}
\tikzset{radiation/.style={{decorate,decoration={expanding
waves,angle=90,segment length=4pt}}}}
\usetikzlibrary{math}
\tikzmath
{
  function symlog(\x,\a){
    \yLarge = ((\x>\a) - (\x<-\a)) * (ln(max(abs(\x/\a),1)) + 1);
    \ySmall = (\x >= -\a) * (\x <= \a) * \x / \a ;
    return \yLarge + \ySmall ;
  };
  function symexp(\y,\a){
    \xLarge = ((\y>1) - (\y<-1)) * \a * exp(abs(\y) - 1) ;
    \xSmall = (\y>=-1) * (\y<=1) * \a * \y ;
    return \xLarge + \xSmall ;
  };
}

\graphicspath{{./figures/}}

\usepackage{scalerel,stackengine}

\makeatletter
\newcommand{\myitem}[1]{%
\item[#1]\protected@edef\@currentlabel{#1}%
}
\makeatother

\newtheorem{theo}{Theorem}
\newtheorem{lem}{Lemma}
\newtheorem{defi}{Definition}
\newtheorem{prop}{Proposition}
\newtheorem{cor}{Corollary}
\newtheorem{remark}{Remark}
\newtheorem{assume}{Assumption}

\newcommand{\fakepar}[1]{\vspace{1mm}\noindent\textbf{#1.}}

\DeclareSIUnit{\belmilliwatt}{Bm}
\DeclareSIUnit{\dBm}{\deci\belmilliwatt}

\DeclareMathOperator*{\E}{\mathbb{E}}

\newcommand{\norm}[1]{\left\lVert#1\right\rVert}
\newcommand{\abs}[1]{\left\lvert#1\right\rvert}

\DeclareMathOperator*{\R}{\mathbb{R}}

\newcommand{\transp}{\text{T}}
\DeclareMathOperator*{\mmd}{\mathrm{MMD}}
\DeclareMathOperator*{\mmdsq}{\mathrm{MMD}^2}

\newcommand{\x}[1]{\tilde{x}_{#1}}

\let\originalleft\left
\let\originalright\right
\renewcommand{\left}{\mathopen{}\mathclose\bgroup\originalleft}
\renewcommand{\right}{\aftergroup\egroup\originalright}

\newcommand\figref[1]{Figure~\ref{#1}}

\newcommand\secref[1]{Section~\ref{#1}}

\newcommand{\eg}{e.g.,\xspace}
\newcommand{\ie}{i.e.,\xspace}

\newcommand{\capt}[1]{\mdseries{\emph{#1}}}

\newcommand{\iid}{i.i.d.\xspace}

\usepackage{ifthen}
\newboolean{authnotes}

\setboolean{authnotes}{true}

\ifthenelse{\boolean{authnotes}}
{
\newcommand{\am}[1]{\footnote{{\bf\color{blue!70!black} Alon: #1}}}
\newcommand{\db}[1]{\footnote{{\bf\color{green!50!black} Dominik: #1}}}
\newcommand{\st}[1]{\footnote{{\bf\color{purple!90!black} Sebastian: #1}}}
\newcommand{\mt}[1]{\footnote{{\bf\color{orange!50!black} Matteo: #1}}}
}
{
\newcommand{\am}[1]{}
\newcommand{\db}[1]{}
\newcommand{\st}[1]{}
\newcommand{\mt}[1]{}
}

\usepackage[hidelinks]{hyperref}
\usepackage{fancyhdr}
\newcommand{\mytitle}{\textbf{Accepted final version.}
To appear in \textit{IEEE Transactions on Robotics, 2024}.\\
\copyright 2024 IEEE. Personal use of this material is permitted. Permission
from IEEE must be obtained for all other uses, in any current or future
media, including reprinting/republishing this material for advertising or
promotional purposes, creating new collective works, for resale or
redistribution to servers or lists, or reuse of any copyrighted component of
this work in other works.}
\begin{document}

\title{Safe reinforcement learning in uncertain contexts}

\author{{Dominik Baumann and Thomas B.~Sch\"{o}n,~\IEEEmembership{Senior Member,~IEEE}}
\thanks{Dominik Baumann is with the Department of Electrical Engineering and Automation, Aalto University, Espoo, Finland, and the Department of Information Technology, Uppsala University, Sweden (e-mail: dominik.baumann@aalto.fi).
Thomas B.\ Sch\"{o}n is with the Department of Information Technology, Uppsala University, Sweden (e-mail: thomas.schon@it.uu.se).

This work has received funding from the Federal Ministry of Education and Research (BMBF) and
the Ministry of Culture and Science of the German State of North Rhine-Westphalia
(MKW) under the Excellence Strategy of the Federal Government and the Länder, the project \emph{NewLEADS - New Directions in Learning Dynamical Systems} (contract number: 621-2016-06079), funded by the Swedish Research Council, and the \emph{Kjell och Märta Beijer Foundation}. A donation by Mitsubishi Electric Research Laboratories (MERL) was used to purchase the equipment used for the experiments.}
}

\markboth{IEEE Transactions on Robotics,~Vol.~14, No.~8, August~2021}%
{Baumann \MakeLowercase{\textit{et al.}}: Safe reinforcement learning in uncertain contexts}

\IEEEpubid{0000--0000/00\$00.00~\copyright~2022 IEEE}

\maketitle

\thispagestyle{fancy}   
\pagestyle{empty}

\begin{abstract}
When deploying machine learning algorithms in the real world, guaranteeing safety is an essential asset.
Existing safe learning approaches typically consider continuous variables, \ie regression tasks.
However, in practice, robotic systems are also subject to discrete, external environmental changes, \eg having to carry objects of certain weights or operating on frozen, wet, or dry surfaces.
Such influences can be modeled as discrete \emph{context} variables.
In the existing literature, such contexts are, if considered, mostly assumed to be known.
In this work, we drop this assumption and show how we can perform safe learning when we cannot directly measure the context variables.
To achieve this, we derive frequentist guarantees for multi-class classification, allowing us to estimate the current context from measurements.
Further, we propose an approach for identifying contexts through experiments.
We discuss under which conditions we can retain theoretical guarantees and demonstrate the applicability of our algorithm on a Furuta pendulum with camera measurements of different weights that serve as contexts.
\end{abstract}


\section{Introduction}
\IEEEPARstart{W}{hen} learning on real and potentially expensive robotics hardware, respecting safety constraints is instrumental.
In response to this need, safe learning algorithms have been developed that provide different forms of safety certificates~\cite{brunke2021safe}.
One such algorithm is \textsc{SafeOpt}~\cite{berkenkamp2021bayesian}, which guarantees to find an optimal policy while never violating any safety constraints during the search with high probability.

Most safe learning algorithms consider only the internal dynamics of a robot (and possibly external constraints).
However, when learning in the real world, robotic systems are subject to changes in their environment that influence their dynamics. 
Consider the illustrative example shown in \figref{fig:exp_setup}.
We have a Furuta pendulum and aim to learn a balancing policy.
Different kinds of weights can be added to the top of the pole.
These weights clearly influence the dynamics of the pendulum.
However, the pendulum cannot control which weight is added.
Including the weight as an external parameter in the policy optimization is then challenging for two reasons.
First, as the pendulum cannot control which weight is added, it cannot actively generate informative data.
Further, we may have only a few examples (\eg only three different weights in \figref{fig:exp_setup}). 
This makes it hard to perform a regression. 
Second, directly including the weights in the policy optimization increases the dimensionality of the parameter space, which can be computationally demanding for safe learning algorithms.
To account for that,~\cite{berkenkamp2021bayesian} proposes such parameters to be included as discrete context variables and shows that \textsc{SafeOpt} can successfully deal with a mixed continuous and discrete parameter space.
Nevertheless, they assume that the context variable, \eg the weight, is known.
In \figref{fig:exp_setup}, we only have a camera that can capture the current weight.
This is a typical industrial setting, where cameras are often used to track work pieces.
Yet, inferring the weight of an object from image data is not possible in general~\cite{standley2017image2mass}.

\begin{figure}
\centering
\input{tikz/context_setup}
\caption{Our experimental setup.
\capt{We aim at optimizing a balancing controller for a Furuta pendulum whose dynamics can be altered by adding (removing) weights to (from) its pole. 
Our algorithm tries to infer the current weight from image data and resorts to identifying it through dedicated experiments if the image data is not sufficiently informative.}}
\label{fig:exp_setup}
\end{figure}

\IEEEpubidadjcol

An alternative to inferring the weight from data would be identifying the context through experiments.
Such an algorithm is proposed in~\cite{achterhold2021explore}.
The downside of this approach is that it requires additional experimentation and cannot provide formal guarantees on identifying the correct context.
Thus, it cannot straightforwardly be combined with a safe learning algorithm.
Suppose we instead understand trajectories as samples from different probability distributions. 
In that case, we can investigate whether the distribution generating the current sample is identical to the distribution of a sample we saw in the past.
One measure for comparing distributions is the maximum mean discrepancy (MMD)~\cite{gretton2012kernel}, whose applicability to dynamical systems has been discussed in~\cite{solowjow2020kernel}.
In this article, we show how we can extend the formal guarantees the MMD provides for context identification in a safe learning algorithm.

Still, the problem of additional experimentation remains.
Thus, over time, we would like to learn a classifier that can, at least in some cases, infer the current weight from image data.
In particular, the classifier should output a probability that the camera image corresponds to a certain weight.
This is a standard multi-class classification setting, for which various approaches exist, \eg support vector machines (SVMs)~\cite{cervantes2020comprehensive}, neural networks~\cite{he2016deep}, or conditional mean embeddings (CMEs)~\cite{hsu2018hyperparameter}.
However, if we want to combine them with a safe learning algorithm, we again require theoretical guarantees, which typical classification approaches cannot provide in the required form.
Thus, we derive frequentist uncertainty bounds for multi-class classification based on CMEs and show how they can be used within a safe learning algorithm.

In summary, we propose an algorithm that mainly consists of two parts.
We propose to train a classifier that yields a probabilistic score to relate, for instance, camera images to weights.
For this classifier, we derive frequentist uncertainty bounds to use its output in an existing safe exploration algorithm, such as \textsc{SafeOpt}, and provide safety guarantees.
If the uncertainty of our classifier is too high, we propose to identify the current context.
Also here, we provide guarantees that can be used in safe exploration.
Further, we update our classifier to sharpen the uncertainty bounds whenever the current context has been identified.

\fakepar{Contributions}
In this paper, we
\begin{itemize}
    \item derive frequentist uncertainty intervals for multi-class classification;
    \item develop a context identification method with statistical guarantees;
    \item combine multi-class classification with frequentist uncertainty intervals and context identification in an algorithm that allows for safe learning in uncertain contexts;
    \item evaluate the algorithm on a Furuta pendulum~\cite{furuta1992swing} with camera images of weights serving as contexts.
\end{itemize}

\section{Related work}
\label{sec:rel_work}

This paper proposes an algorithm for learning safe optimal policies in context-dependent dynamical systems with uncertain contexts.
Here, we relate our contribution to the literature.

\fakepar{Safe Learning}
For a general overview of safe learning, we refer the reader to~\cite{brunke2021safe}.
Most of these works focus on pure regression tasks without discrete parameters.
Introducing discrete parameters as context variables initially emerged in the bandit setting~\cite{krause2011contextual,langford2007epoch,wang2005bandit,lai1995machine}. 
In~\cite{berkenkamp2021bayesian}, it was shown that the concept of discrete contexts could also benefit safe policy optimization for dynamical systems.
In~\cite{berkenkamp2021bayesian} and in the works on bandit settings, it is typically assumed that the current context is known to the learning agent.
In reality, this may not always be the case.
Thus, we here consider a setting where the agent receives some sensor information that it can use to infer the context, but it cannot directly measure the current context.
Another approach that considers unknown contexts in a Bayesian framework is presented in~\cite{feng2020high}.
However, the authors assume no information about the context, which differs from our setting, and they implicitly model the context. 
In contrast, we explicitly account for it in the learning problem.

\fakepar{Learning with unobserved context}
Some works consider optimizing policies for context-dependent dynamical systems with unobserved contexts~\cite{feng2020high,swamy2022sequence}.
The setting we consider is different in that we assume some information about the context to be available even though it cannot be unambiguously determined through the measurements.
Other approaches, such as~\cite{konig2021safe}, assume access to some low-dimensional task parameter that is directly correlated with the context.
This allows them to integrate the task parameter directly into the learning algorithm.
In~\cite{konig2021safe}, the authors further assume that they can observe the task parameter over a range of values.
Neither of these assumptions are satisfied in our setting.
We assume access to a high-dimensional observation, such as image data, that is related to the context.
Thus, we cannot directly include the observation in the learning algorithm due to the scalability properties of Gaussian process regression.
Moreover, we consider a setting in which we receive discrete observations, such as images of traffic signs, where learning a classifier is more efficient than trying to learn a continuous model over all possible pixel values.
In control theory, discrete contexts that alter the dynamics of a system may be interpreted as disturbances.
Assuming knowledge of worst-case bounds for those disturbances, we can design robust controllers that provide stability guarantees while sacrificing performance~\cite{zhou1996robust}.
We do not assume knowledge of worst-case bounds. 
Instead, we assume that there are measurements that can be used to estimate the current context.
This idea is similar to disturbance observers that try to reconstruct the disturbance given the measurements and then compensate for them~\cite{chen2015disturbance}.
Both robust control techniques and disturbance observers require access to a mathematical model of the underlying system, while the approach we present herein is model-free.

\fakepar{Context identification}
A dedicated context identification algorithm is proposed in~\cite{achterhold2021explore}.
Nevertheless, it lacks theoretical guarantees.
Thus, we use the MMD initially introduced in~\cite{gretton2012kernel} and extended in~\cite{solowjow2020kernel} to dynamical systems.
Our main idea is that contexts change the dynamics of the system. 
Thus, by comparing the current trajectory with trajectories collected in the past, we can infer the current value of the context variable.
This boils down to comparing the probability distributions generating the trajectories.
For this, also other measures than the MMD are of course possible; see~\cite {sriperumbudur2012empirical} for an overview.
We here choose the MMD since it allows us to compare probability distributions without actually estimating them, provides theoretical guarantees, and can be computed efficiently.

\fakepar{Classification}
While various algorithms for classification exist, we here require one that provides input-dependent frequentist uncertainty intervals.
This is especially important since classifiers tend to be over-confident~\cite{bai2021don,kull2019beyond}.
In the literature, we find general bounds for multi-class classification~\cite{hsu2018hyperparameter,lei2019data}, probably approximately correct (PAC) bounds for Gaussian process classification~\cite{seeger2002pac}, bounds for distribution-free calibration in binary classification~\cite{gupta2020distribution}, for general CMEs~\cite{park2020measure,smale2007learning}, as well as their application to classification~\cite{hsu2018hyperparameter}.
Nevertheless, all these bounds have in common that they are \emph{uniform}.
While uniform bounds are strong mathematical statements, they are not the most suitable objects for our problem setting.
In particular, they typically only hold for the misclassification risk averaged over the entire input space.
Instead, we require concrete statements about the misclassification risk at a particular input location.
Intuitively, we should be more confident about our estimates in regions of the input space where we have already seen many data points.
This is also pointed out as a potential extension in the appendix of~\cite{hsu2018hyperparameter}.
Further, existing learning theoretic bounds often aim to provide convergence rates, and actually applying those bounds to practical settings with finitely many data points is non-trivial.
A main contribution of our work is the derivation of finite-sample frequentist bounds for classification with CMEs.
Another approach that considers input-dependent convergence bounds is presented in~\cite{maddalena2021deterministic}.
In that work, the authors provide deterministic uncertainty intervals for kernel ridge and support vector regression, although not aiming at classification settings. 
In \secref{sec:class_eval}, we discuss how their bounds can be applied to classification and why they are non-informative in that setting.

\section{Problem setting and background}

We consider a setting in which an RL agent seeks to optimize a policy but needs to guarantee that it satisfies safety guarantees throughout exploration.
Further, the dynamics are influenced by an unobserved context variable.
Our main contributions are a context identification method with statistical guarantees and frequentist uncertainty intervals for multi-class classification that together allow the context to be inferred from external measurements.
We propose to combine these two contributions with a safe learning algorithm.
The concrete safe learning algorithm that is actually in charge of optimizing the policy is, thus, \emph{independent} of our contributions, and basically, any RL algorithm that provides safety guarantees could be used.
Nevertheless, to make the problem setting more concrete, we will here consider the algorithm from~\cite{berkenkamp2021bayesian} and introduce this setting in \secref{sec:background}.
The problem that we address in this paper is formulated in \secref{sec:probl_setting}.

\subsection{Background}
\label{sec:background}

We consider a dynamic system
\begin{equation}
    \label{eqn:sys}
    x(k+1) = z(x(k),u(k),c)
\end{equation}
with discrete time index $k\in\mathbb{N}$, state $x(k)\in\mathcal{X}\subseteq\R^\ell$, and input $u(k)\in\R^m$, whose dynamics depend on a context parameter $c\in\mathcal{C}\subseteq\mathbb{N}$.
For this system, we want to learn a policy $u(k)=\pi(x(k), c, a)$, parametrized by parameters $a\in\mathcal{A}\subseteq\R^d$ that maximizes an unknown reward function $f:\, \mathcal{A}\times\mathcal{C}\to\R$ while guaranteeing safety.
Safety is encoded through (unknown) constraint functions $g_i:\,\mathcal{A}\times\mathcal{C}\to\R$, $i\in\{1,\ldots,q\}$.

We assume that the reward and constraint functions are unknown. 
However, we can receive (noisy) measurements of both by doing experiments: we select a parameterization $a$, perform an experiment, and afterward receive measurements of $f$ and $g_i$ for all $i\in\{1,\ldots,q\}$.
That way, we can, over time, find the optimum of $f$.
Importantly, we seek to provide safety guarantees for each exploration experiment.
Thus, the overall optimization problem is
\begin{equation}
\label{eqn:safeopt}
    \max_{a\in\mathcal{A}} f(a,c) \text{ s.t. } g_i(a,c)\ge 0 \text{ for all }i\in\{1,\ldots,q\}. 
\end{equation}
We need a few assumptions to enable safe exploration despite unknown dynamics and constraints~\cite{berkenkamp2021bayesian}.
\begin{assume}
\label{ass:safeopt_rkhs_norm}
The reward function $f$ and the constraint functions $g_i$, with $i\in\{1,\ldots,q\}$, have bounded norm in a reproducing kernel Hilbert space (RKHS).
\end{assume}
\begin{assume}
\label{ass:safeopt_meas_noise}
After each exploration experiment, we receive noisy measurements of the reward and constraint functions.
Those measurements are perturbed by $\sigma$-sub Gaussian measurement noise.
\end{assume}
\begin{assume}
\label{ass:safeopt_starting_point}
We are given at least one safe parameter vector $a$, for which we have, for all contexts $c\in\mathcal{C}$ and for all $i\in\{1,\ldots,q\}$, that $g_i(a,c)>0$ with probability at least $1-\delta_\mathrm{safe}$.
\end{assume}
Assumption~\ref{ass:safeopt_starting_point} may seem relatively strong as it requires parameter vectors that are safe under all contexts. 
It is required since we need some policy to start with, and at least during the first experiment, we cannot estimate the context without prior knowledge.
In practice, such an initial safe parameter can have arbitrarily bad performance.
For instance, when considering a mobile robot supposed to reach some target while not colliding with obstacles where the contexts are different surfaces, an initial safe policy could barely move the robot.
Such a policy would be safe but have a close to minimum reward.
Still, it would be enough as a starting point for our algorithm.

\subsection{Problem setting}
\label{sec:probl_setting}
In this work, we drop the often-adopted assumption that the discrete context parameter $c$ is known or can be precisely measured.
Instead, we assume that we receive measurements $y\in\mathcal{Y}\subseteq\R^s$ that reveal information about the context.
These could be temperature measurements that induce a probability of whether or not the road is frozen for an autonomous car or camera images that allow us to reason about the weight of an object that a robot is supposed to manipulate.
To automate this reasoning while still guaranteeing safety as above, we need an efficient way to estimate the probability of being in a particular context that itself also provides guarantees.

We define the underlying probability space as $(\Omega, \mathcal{F}, P)$ with random variables $Y:\,\Omega\to\mathcal{Y}$ and $C:\,\Omega\to\mathcal{C}$ that take values in $\mathcal{Y}$ and $\mathcal{C}$, respectively and whose respective probability distributions we denote by $P_Y$ and $P_C$.
In this work, we aim at estimating the probability $p_{c}(y)\coloneqq \mathbb{P}(C=c\mid Y=y)$, \ie the probability of context $c\in\mathcal{C}\subset\mathbb{N}$ given the current measurements $y\in\mathcal{Y}\subseteq\R^s$.
Furthermore, we aim at deriving high probability uncertainty bounds for the estimate $\hat{p}_c$.
Suppose we are given $n$ tuples $(y, c)$ with $n>0$.
For each context $c$, we generate a vector $\textbf{c}\in\R^{n}$, where entry $c_j$ is 1 if in the corresponding tuple of sample $j$ the context was $c_j$ and 0 otherwise.
We aim at finding upper bounds $\epsilon_c(y, \delta,n)$ such that for each $c$ we have with probability at least $1-\delta_\mathrm{class}$, where $\delta_\mathrm{class}\in(0, 1)$,
\begin{align}
\label{eqn:prob_setting}
\abs{p_c(y)-\hat{p}_c(y)} \le \epsilon_c(y, \delta,n)
\end{align}
at the current input $y$.

However, in the beginning, when we have not gathered any data, the bounds~\eqref{eqn:prob_setting} may be arbitrarily large and not enable confident classification decisions and, hence, meaningful safety guarantees.
Thus, we additionally require a possibility to \emph{identify} the current context.
As with any identification procedure based on finitely many data points, this identification will come with some uncertainty.
Suppose the current context is $c^*$.
We then seek to guarantee that with probability at least $1-\delta_\mathrm{MMD}$, we have $\hat{c}=c^*$, where $\hat{c}$ is the identified or estimated context.
By enhancing the data set of $(y,c)$ tuples with these estimated contexts $\hat{c}$, we can then, over time, create a data set that we can use to train a classifier.
Nevertheless, this classifier must acknowledge that each $\hat{c}$ is only correct with probability at least $1-\delta_\mathrm{MMD}$.
Thus, we are tackling two problems:
\begin{enumerate}
    \item Show that we can guarantee safety when we need to identify the context;
    \item Incorporate uncertainty about identified contexts into the classifier and show that we can guarantee safety when our classification algorithm is confident enough about its decision.
\end{enumerate}

\section{Preliminaries}

Having introduced the problem setting, we now present the required mathematical foundations to develop the safe learning algorithm.

\subsection{Context identification}
\label{sec:cont_id_prel}
We need to identify the context whenever the probability estimate~\eqref{eqn:class_prob_cme} for the current context is too low or too uncertain.
For this, we collect trajectory data for each context that we encounter.
Then, if we need to identify the current context, we use the safe policy from Assumption~\ref{ass:safeopt_starting_point} to excite the system and compare the generated data with all collected trajectories.
To compare trajectory data, we use the MMD~\cite{gretton2012kernel}.
Given a stored data set $X_c$ of context $c$ and trajectory data $X$ of the current context, we can calculate a finite-sample approximation of the squared MMD as
\begin{align*}
    &\mmdsq(X, X_c)= \frac{1}{r^2}\sum_{i,j=1}^r k_\mathrm{mmd}(X_i, X_j)\\ 
    &+ \frac{1}{r^2}\sum_{i,j=1}^rk_\mathrm{mmd}(X_{c_i},X_{c_j})- \frac{2}{r^2}\sum_{i=1}^r\sum_{j=1}^rk_\mathrm{mmd}(X_i,X_{c_j}),
\end{align*}
with $k_\mathrm{mmd}(\cdot,\cdot)$ a characteristic kernel and $r$ the length of the two data sets.
For $r\to\infty$ and if the trajectory data were independent and identically distributed (\iid), we could now guarantee that $\mmdsq(X, X_c)=0$ if, and only if, data samples $X$ and $X_c$ were generated in the same context~\cite{gretton2012kernel}.
However, data generated by dynamical systems are naturally non-i.i.d.
To arrive at similar statements for our setting, we require an assumption on the trajectories created in each context.
\begin{assume}
\label{ass:mixing}
For any context $c\in\mathcal{C}$, $X_\mathrm{c}$ is stationary and there exists a time shift $a^*$ and a threshold $\kappa(\epsilon,r)$ such that for $X_c=\left(x_c(a^*), x_c(2a^*), \ldots, x_c(na^*)\right)$, we have
\[
\mathbb{P}[\mmdsq(X_c, \bar{X}_c)\ge\kappa] < \epsilon,
\]
where $\bar{X}_c$ is data from an independent trajectory.
\end{assume}
The intuition behind this assumption is that if we sub-sample from the trajectory, we end up with approximately i.i.d.\ data. 
For a more detailed discussion on estimating $a^*$ from data, we refer the reader to~\cite{solowjow2020kernel}.
In there, the authors also show empirically that Assumption~\ref{ass:mixing} seems to be satisfied for human walking.

\subsection{Classification}
\label{sec:prel_class}

For classification, we leverage the concept of CMEs.
We first define positive definite kernel functions $k:\,\mathcal{Y}\times\mathcal{Y}\to\R$ and $\ell:\,\mathcal{C}\times\mathcal{C}\to\{0,1\}$ for the input space (the measurements $y$) and the output space (the contexts $c$), respectively.
In our setting, it is natural to choose $\ell$ as the Kronecker delta kernel.
That is, we assume equal contexts have unit similarity while different contexts have no similarity.
The Kronecker delta kernel is integrally strictly positive definite on $\mathcal{C}$ and, therefore, characteristic~\cite[Thm.~7]{sriperumbudur2010hilbert}.
For the input space, we choose the Gaussian kernel, which is also characteristic, and then have that both $k$ and $\ell$ uniquely define the RKHSs $\mathcal{H}_k$ and $\mathcal{H}_\ell$.

Following~\cite{song2009hilbert}, the conditional kernel mean embedding operator $\mathcal{U}_{C\mid Y=y}$ is the operator $\mathcal{U}:\,\mathcal{H}_k\to\mathcal{H}_\ell$ and the CME is $\mu_{C\mid Y=y}=\mathcal{U}k(y,\cdot)\coloneqq \E[\ell(C,\cdot)\mid Y=y]$.
We further define the cross-covariance operators $R_{CY}\coloneqq\E[\ell(C,\cdot)\otimes k(Y,\cdot)]:\,\mathcal{H}_k\to\mathcal{H}_\ell$ and $R_{YY}\coloneqq\E[k(Y,\cdot)\otimes k(Y,\cdot)]:\,\mathcal{H}_k\to\mathcal{H}_k$.
Given that $k(y,\cdot)$ is in the image of $R_{YY}$, we now have that $\mathcal{U}_{C\mid Y}=R_{CY}R_{YY}^{-1}$.
However, since this assumption is not necessarily satisfied for continuous input spaces $\mathcal{Y}$~\cite{fukumizu2004dimensionality}, with which we generally deal in our problem setting, we instead use the regularized version $\mathcal{U}_{C\mid Y} = R_{CY}(R_{YY}+\lambda I)^{-1}$ with regularization parameter $\lambda$ and $I$ the identity matrix of appropriate dimensions.

Ultimately, we seek to infer the classification probabilities $p_{c}(y)$, \ie the probability of context $c$ given specific measurements $y$.
We can write this probability in terms of the indicator function $\mathbbm{1}(\cdot)$ as
\begin{equation}
\label{eqn:class_prob}
p_c(y) \coloneqq \mathbb{P}(C=c\mid Y=y) = \E[\mathbbm{1}_c(C)\mid Y=y].
\end{equation}
For the Kronecker delta kernel, we have $\mathbbm{1}_c(y)=\ell(c,y)$.
Thus, the indicator function $\mathbbm{1}_c=\ell(c,\cdot)$ is the canonical feature map of $\mathcal{H}_\ell$ and we can estimate it using the CME~\cite{hsu2018hyperparameter}:
\begin{align}
\label{eqn:class_prob_cme}
\begin{split}
\E[\mathbbm{1}_c(C)\mid Y=y]&\approx \langle\hat{\mu}_{C\mid Y=y},\mathbbm{1}\rangle_k\\
&=\mathbbm{1}^\transp(K+n\lambda I)^{-1}K_y\eqqcolon\hat{p}_c(y),
\end{split}
\end{align}
with $n$ being the number of data points, $K_y\coloneqq\begin{pmatrix}k(y,y_1),\ldots,k(y,y_n)\end{pmatrix}$, $\mathbbm{1}\coloneqq{\mathbbm{1}_c(c_j)}_{j=1}^n$, and the entry $(a,b)$ in the matrix $K\in\R^{n\times n}$ is $k(y_a,y_b)$.
It then also follows that we can essentially estimate the probabilities for all contexts using kernel ridge regression~\cite{hsu2018hyperparameter}.
\begin{theo}[\hspace{1sp}{\cite[Thm.~1]{hsu2018hyperparameter}}]
\label{thm:consistent}
The classifier~\eqref{eqn:class_prob_cme} is consistent if $k(y,\cdot)$ is in the image of $R_{YY}$.
\end{theo}

\begin{remark}
For finite sample sizes,~\eqref{eqn:class_prob_cme} may yield context probabilities above one or below zero, which can be avoided by applying a normalization~\cite{hsu2018hyperparameter}.
\end{remark}

To obtain practically useful bounds, we make an assumption about the regularity of the true context probability functions $p_c(y)$ that is similar to Assumption~\ref{ass:safeopt_rkhs_norm} and generally common in the safe learning literature~\cite{berkenkamp2021bayesian,fiedler2021practical,chowdhury2017kernelized}.
\begin{assume}
\label{ass:class_rkhs_bound}
The true probability functions $p_c(y)$ are all in the Hilbert space $\mathcal{H}_k$ and have bounded norm, $\norm{p_c(y)}_k^2\le \Gamma$ for all $c\in\mathcal{C}$ with known bound $\Gamma$.
\end{assume}

\section{Safe reinforcement learning\\ in uncertain contexts}

We now present our safe reinforcement learning algorithm.
We start by deriving high-probability guarantees for context identification.
Then, we develop frequentist uncertainty intervals for multi-class classification based on CMEs.
Lastly, we integrate both ingredients into the safe learning algorithm.
Before starting, we make the notion of a context more precise.
\begin{defi}
\label{def:mmd}
For any two contexts $c_\mathrm{a}\neq c_\mathrm{b}$, we have that $\mmdsq(X_\mathrm{a}, X_\mathrm{b}) > \eta$ in the large sample limit (\ie when the number of data points $r\to\infty)$.
\end{defi}
That is, we define contexts as external environmental changes that cause a significant change in the system dynamics.
Since the definition via the MMD might seem abstract, we make the notion explicit in Appendix~\ref{sec:mmd_example} for the Furuta pendulum with weights used in the evaluation in \secref{sec:eval}.

\subsection{Context identification with guarantees}
\label{sec:cont_id}

The MMD, as presented in \secref{sec:cont_id_prel}, provides guarantees in the infinite sample limit.
For finitely many data samples,~\cite{gretton2012kernel} presents various test statistics that can be used for hypothesis testing.
Those hypothesis tests come with two challenges for our setting.
First, they assume data to be i.i.d.
Data drawn from dynamical systems is naturally correlated and, thus, not i.i.d.
Second, the hypothesis test can only reject the null hypothesis $c_\mathrm{a}=c_\mathrm{b}$ based on a chosen significance level but cannot provide guarantees for detecting that $c_\mathrm{a}\neq c_\mathrm{b}$.
In fact,~\cite{gretton2012kernel} shows through an example that providing such guarantees for distinguishing probability distributions is generally impossible, independent of the distance measure.

We address the first challenge by employing the sub-sampling strategy from~\cite{solowjow2020kernel}.
That is, from every collected trajectory, we sub-sample data such that the time shift equals $a^*$ from Assumption~\ref{ass:mixing}.
For the second challenge, we leverage Definition~\ref{def:mmd}.
Then, we arrive at the following statement.
\begin{prop}
\label{prop:mmd}
Under Assumption~\ref{ass:mixing} and given $r$ data samples $X$ and $X_c$, sub-sampled from the whole trajectory such that the time shift is $a^*$.
Set 
\[
\eta = 4\sqrt{\frac{2K}{r}}\left(1+\sqrt{2\ln\frac{2}{\delta_\mathrm{MMD}}}\right)
\]
in Definition~\ref{def:mmd}.
Assuming a characteristic kernel $k_\mathrm{MMD}$ with $0\le k_\mathrm{MMD}(x,y)\le K$, we have, with probability at least $1-\delta_\mathrm{MMD}$,
\begin{align*}
\mmdsq(X, X_c) < 2\sqrt{\frac{2K}{r}}\left(1+\sqrt{2\ln\frac{2}{\delta'_\mathrm{MMD}}}\right),
\end{align*}
where $\delta_\mathrm{MMD}=\frac{1}{3}(\delta'_\mathrm{MMD} + 2\epsilon)$ with $\epsilon$ from Assumption~\ref{ass:mixing},
if, and only if, the probability distribution that generated the data of the current context is the same that generated the trajectory data of context $c$.
\end{prop}
\begin{proof}
Follows from combining~\cite[Thm.~3]{solowjow2020kernel} with Definition~\ref{def:mmd}.
\end{proof}
That is, the guarantee for context identification we provide is essentially the same as in~\cite{solowjow2020kernel}, except that we can also guarantee to detect $c_\mathrm{a}\neq c_\mathrm{b}$.
This is possible because of Definition~\ref{def:mmd}.
Since different contexts, by definition, change the dynamics significantly, we can guarantee to detect this change with the proposed hypothesis test.
In practice, we might also have external changes in the environment that only cause minor changes in the system dynamics.
Then, learning a new policy would be inefficient, and it is more sensible to consider this still the same context.
The design parameter $\eta$ in Definition~\ref{def:mmd} quantifies when we consider an environmental change as \emph{significant}, and we provide some intuition for it in Appendix~\ref{sec:mmd_example}.

\subsection{A classifier with frequentist bounds}
\label{sec:class}

Our goal is to develop a safe learning algorithm in uncertain contexts where we perform context identification only if necessary since it consumes time and causes wear and tear to the hardware.
If we assume no prior knowledge about how measurements relate to contexts, we have, in the beginning, no other choice than to identify contexts.
Over time, we can then learn a model that relates those identified contexts to received measurements.
This is a standard classification setting.
If we want to include the classification in the safe learning algorithm, it has to provide frequentist guarantees. 

For classification, we need to consider three types of uncertainties: \emph{(i)} uncertainty from estimating a function with limited amount of data, \emph{(ii)} uncertainty from not obtaining samples of the true probability function $p_c(y)$ but only discrete labels, \emph{(iii)} uncertainty that stems from context identification which provides us with the correct context with probability $1-\delta_\mathrm{MMD}$ (see Proposition~\ref{prop:mmd}).
In the following, we show how we can bound all three types of uncertainties and then combine them to obtain the required overall frequentist guarantees.

First, we define a variant of the \emph{power function}~\cite{maddalena2021deterministic}, which will occur at several stages during the derivations.
\begin{defi}
\label{def:power_function}
The power function is defined as $\varrho(y) \coloneqq \sqrt{k(y,y)-K_y^\transp(K+n\lambda I)^{-1}K_y}$.
\end{defi}

We now first address uncertainty~\emph{(i)} by introducing a virtual estimate $\bar{p}_c$ of $\hat{p}_c$.
This virtual estimate is the estimate we would get if we could use measurements of the true context probability function $p_c(y)$ in~\eqref{eqn:class_prob_cme} instead of discrete labels.
\begin{lem}
\label{lem:est_uncertainty}
Under Assumption~\ref{ass:class_rkhs_bound}, and given perfect measurements $\textbf{p}_\textbf{y}\in\R^n$ of $p_c(y)$, for any $n>0$, we have for all $y\in\mathcal{Y}$ and $c\in\mathcal{C}$
\begin{equation*}
\abs{p_c(y)-\bar{p}_c(y)} \le \sqrt{\Gamma}\varrho(y),
\end{equation*}
where $\bar{p}_c$ is estimated using $\textbf{p}_\textbf{y}$ and with $\varrho(y)$ the power function from Definition~\ref{def:power_function}.
\end{lem}
\begin{proof}
Similar in nature to that of~\cite[Thm.~2]{chowdhury2017kernelized}.
Details are provided in Appendix~\ref{sec:proof_lem_1}.
\end{proof}

However, in practice, we only receive discrete labels.
Thus, we next analyze the uncertainty from not measuring $p_c(y)$.
We first establish that centered Bernoulli random variables are $\sigma$-sub-Gaussian.
\begin{lem}
\label{lem:subGaussian}
Let $c$ be a Bernoulli random variable with success probability $p_c$.
We have that the random variable $c-p_c$ is $\sigma$-sub Gaussian with $\sigma \le \frac{1}{4}$.
\end{lem}
\begin{proof}
Follows from Theorem~2.1 and Lemma~2.1 in~\cite{buldygin2013sub}.
\end{proof}

For bounding the measurement uncertainty, we similarly introduce a virtual estimate $\check{p}_c$, which is the estimate we would get if the context identification always returned the actual context.

\begin{lem}
\label{lem:meas_uncertainty}
We have for any $n>0$ and all $y\in\mathcal{Y}$, with probability at least $1-\delta$,
\begin{equation*}
\abs{\bar{p}_c(y)-\check{p}_c(y)} \le \frac{\varrho(y)}{4\sqrt{n\lambda}}\sqrt{\log(\det(K + \bar{\lambda}I)) - 2\log(\delta)},
\end{equation*}
with $\bar{p}_c(y)$ as in Lemma~\ref{lem:est_uncertainty}, $\varrho(y)$ from Definition~\ref{def:power_function}, and $\bar{\lambda} \coloneqq\max\{1, n\lambda\}$. 
\end{lem}
\begin{proof}
The proof idea is similar to that of~\cite[Thm.~1]{fiedler2021practical}, which is possible since the measurement uncertainty is $\sigma$-sub Gaussian with $\sigma\le\frac{1}{4}$ by Lemma~\ref{lem:subGaussian}.
Details are provided in Appendix~\ref{sec:proof_lem_2}.
\end{proof}

Lastly, we address the uncertainty~\emph{(iii)} inherited through context identification.
\begin{lem}
\label{lem:class_cont}
Given the setting in Proposition~\ref{prop:mmd}, we have, with probability at least $1-\delta_\mathrm{MMD}$,
\begin{align*}
&\abs{\check{p}_c(y) - \hat{p}_c(y)} \le \\
&\frac{\varrho(y)(1-2\delta_\mathrm{MMD})\sqrt{\log(\det(K + \bar{\lambda}I)) - 2\log(\delta_\mathrm{MMD})}}{2(\ln(1-\delta_\mathrm{MMD})-\ln(\delta_\mathrm{MMD}))\sqrt{n\lambda}}\\
&+ (1-\delta_\mathrm{MMD})\mathbf{1}(K+n\lambda I)^{-1}K_y,
\end{align*}
with $\check{p}_c(y)$ as in Lemma~\ref{lem:meas_uncertainty}, $\varrho(y)$ from Definition~\ref{def:power_function}, and $\bar{\lambda} = \max\{1, n\lambda\}$. 
\end{lem}
\begin{proof}
The two quantities are identical except for a potential mismatch of actual context $c$ and estimated context $\hat{c}$.
We thus analyze the error $\abs{c-\hat{c}}$, which we can rewrite as $\abs{c-\hat{c}-(1-\delta_\mathrm{MMD}) + (1-\delta_\mathrm{MMD})}\le \abs{(c - \hat{c} - (1-\delta_\mathrm{MMD})} + \abs{1-\delta_\mathrm{MMD}}$.
Following~\cite[Lem.~2.1]{buldygin2013sub}, the first term is a sub-Gaussian random variable with $\sigma=\frac{1-2\delta_\mathrm{MMD}}{2(\ln(1-\delta_\mathrm{MMD})-\ln(\delta_\mathrm{MMD}))}$.
Thus, we can bound the error term in the same way as shown in Lemma~\ref{lem:meas_uncertainty}.
\end{proof}

Combining the lemmas, we arrive at the desired statement.
\begin{cor}
\label{cor:prob_uncertainty}
Under Assumption~\ref{ass:class_rkhs_bound} and given the setting in Proposition~\ref{prop:mmd}, we have, with probability at least $(1-\delta_\mathrm{MMD})(1-\delta_\mathrm{class})$,
\begin{align*}
&\abs{p_c(y)-\hat{p}_c(y)} \le\\ 
&\varrho(y)\left(\sqrt{\Gamma}+\frac{1}{4\sqrt{n\lambda}}\sqrt{\log(\det(K + \bar{\lambda}I)) - 2\log(\delta_\mathrm{class})}\right.\\
&+\left.\frac{(1-2\delta_\mathrm{MMD})\sqrt{\log(\det(K + \bar{\lambda}I)) - 2\log(\delta_\mathrm{MMD})}}{2(\ln(1-\delta_\mathrm{MMD})-\ln(\delta_\mathrm{MMD}))\sqrt{n\lambda}}\right)\\
&+ (1-\delta_\mathrm{MMD})\mathbf{1}(K+n\lambda I)^{-1}K_y,
\end{align*}
with $\varrho(y)$ from Definition~\ref{def:power_function} and $\bar{\lambda} = \max\{1, n\lambda\}$, holds for any $n>0$ and all $y\in\mathcal{Y}$.  
\end{cor}

\begin{remark}
So far, we assumed that the number of contexts $\abs{\mathcal{C}}=m$ is given a priori.
Accounting for potential unknown contexts is straightforward.
In that case, we consider the number of contexts $m_n$ as a variable that can change as we gather more data (more $(y,c)$ tuples).
Whenever we receive a previously unseen context, we increase $m_{n+1}=m_n+1$.
Since we have not seen this context before, we can create its measurement vector as a vector of all zeros except for a one in the last entry.
\end{remark}

\subsection{Safe learning}
\label{sec:safe_learning}

Finally, we show how all previous results can be merged into a safe learning algorithm (see Algorithm~\ref{alg:safe_learning}).

\begin{algorithm}
\small
\caption{Pseudocode of the safe reinforcement learner.}
\label{alg:safe_learning}
\begin{algorithmic}[1]
\State \textbf{Input:} Measurements $y$, safety threshold $p_\mathrm{safe}$
\For {$c\in\mathcal{C}$}
\State Compute $\hat{p}_c(y)$ using~\eqref{eqn:class_prob_cme}
\State Estimate $\Delta p_c(y) = \abs{\hat{p}_c(y)-p_c(y)}$ with Cor.~\ref{cor:prob_uncertainty}
\If {$\hat{p}_c(y)-\Delta p_c(y)>p_\mathrm{safe}$}
\State \textbf{Return:} Context $c$
\Else
\State Perform experiment, measure $X$ trajectory
\For {$c\in\mathcal{C}$}
\If {$\mmdsq(X, X_c)$ below threshold from Prop.~\ref{prop:mmd}}
\State \textbf{Return:} Context $c$
\EndIf
\EndFor
\EndIf
\EndFor
\State \textbf{Return:} Context $c\notin\mathcal{C}$
\end{algorithmic}
\end{algorithm}

We need one more assumption before analyzing the safety of Algorithm~\ref{alg:safe_learning}.
We assume that all relevant safety information that $y$ can provide is encoded in the context.
\begin{assume}
\label{ass:cont_full_inf}
Let $P(\text{safe})$ denote the probability of an experiment being safe.
We have $P(\mathrm{safe}\mid c, y) = P(\mathrm{safe}\mid c)$.
\end{assume}

Then, the safety of Algorithm~\ref{alg:safe_learning} can be formalized as follows.

\begin{theo}
\label{thm:safety}
Given Assumptions~\ref{ass:safeopt_rkhs_norm}--\ref{ass:cont_full_inf}.
Then, following Algorithm~\ref{alg:safe_learning}, for any $n\ge 0$, the experiment is safe with probability at least $(1-\delta_\mathrm{safe})(p_\mathrm{safe})(1-\delta_\mathrm{class})(1-\delta_\mathrm{MMD})$.
\end{theo}
\begin{proof}
We distinguish two cases.
In the first case, we need to identify the context.
Then, we have
\begin{align*}
P(\mathrm{safe}\mid c)&=P(\mathrm{safe}\mid \hat{c}=c)P(\hat{c}=c)\\
&\ge(1-\delta_\mathrm{safe})(1-\delta_\mathrm{MMD})
\end{align*}
by~\cite[Thm.~1]{berkenkamp2021bayesian} and Proposition~\ref{prop:mmd}.
In the second case, we are certain enough about the current context to choose a policy directly.
By Corollary~\ref{cor:prob_uncertainty}, we can bound the uncertainty of our context inference.
Thus, we have
\begin{align*}
P(\text{safe}\mid c) &= P(\text{safe}\mid \hat{c}=c)P(c=\hat{c}\mid y)\\
&\ge (1-\delta_\mathrm{safe})P(c=\hat{c}\mid y) \\
&\ge (1-\delta_\mathrm{safe})p_\mathrm{safe}(1-\delta_\mathrm{class})(1-\delta_\mathrm{MMD}).
\end{align*}
\end{proof}

\section{Evaluation}
\label{sec:eval}

We evaluate our algorithm using the scenario shown in \figref{fig:exp_setup}.
After evaluating the algorithm as a whole, we provide a comparison with the standard \textsc{SafeOpt} algorithm without our contributions and then specifically investigate the performance of the classification bounds.

\subsection{Safe reinforcement learning in uncertain contexts}
\label{sec:eval_full}
In \figref{fig:exp_setup}, we have a Furuta pendulum~\cite{furuta1992swing} to whose pole we can add weights.
We also have a camera that can take a picture of the current weight, allowing us to infer which weight was added before an experiment.
We consider learning a balancing controller for the Furuta pendulum using the \textsc{SafeOpt}~\cite{berkenkamp2021bayesian} algorithm.
In particular, we consider linear state-feedback control, \ie we multiply the four-dimensional state, consisting of angle $\alpha$ and angular velocity $\dot{\alpha}$ of the rotatory arm, and angle $\theta$ and angular velocity $\dot{\theta}$ of the pole, with a feedback matrix $F\in\R^{1\times 4}$.
We then focus on letting \textsc{SafeOpt} find an optimal value for the feedback gain multiplied by the pole angle while keeping the other entries of $F$ fixed.
During the search, \textsc{SafeOpt} shall avoid failures with high probability.
Here, we define failure as the pole dropping.
For interfacing the pendulum system, we leverage code provided with~\cite{bleher2022learning}.

For \textsc{SafeOpt}, we use a Matern kernel for the parameter optimization with a length scale of 0.1 and a Gaussian kernel with a length scale of 1 for the contexts.
Apart from that, we leave the hyperparameters provided in the official code~\cite{berkenkamp2021bayesian} untouched.
For classification, we choose, inspired by the classification example in~\cite[Ch.~3]{williams2006gaussian}, a Gaussian kernel with a log length scale of 7.5 and log magnitude of 1.5.
We further set $\lambda=\num{1e-4}$ and $\Gamma=2$.
We keep those two parameters for all experiments.
In the following section, we discuss and numerically estimate the value of $\Gamma$ for a different setting.
From this discussion, we conclude that $\Gamma=2$ is a sensible choice for making sure to stay safe and not allow for failures while at the same time not being too conservative.
For the context identification, we also choose a Gaussian kernel and compute the length scale based on the data samples as suggested in~\cite{gretton2012kernel}.

At the beginning of each experiment, we let a random number generator determine the current context and, accordingly, add one of the two weights or no weight to the tip of the pole.
We then take an image of the weight using a standard smartphone camera as is shown in \figref{fig:exp_setup}.
We convert each image to grayscale and scale it to a size of $32\times32$ pixels.
For this rescaled image, we then compute the classification bounds from Corollary~\ref{cor:prob_uncertainty}.
Nevertheless, as we assume no prior knowledge, the uncertainty bounds are high during the first iterations.
Thus, we must always identify the current context in the first iterations.
In these cases, we use the initially given safe controller and excite the system by adding a chirp signal.
We show example trajectories for the two weights and one without any weight in \figref{fig:trajectories}.

\begin{figure*}
\centering
\input{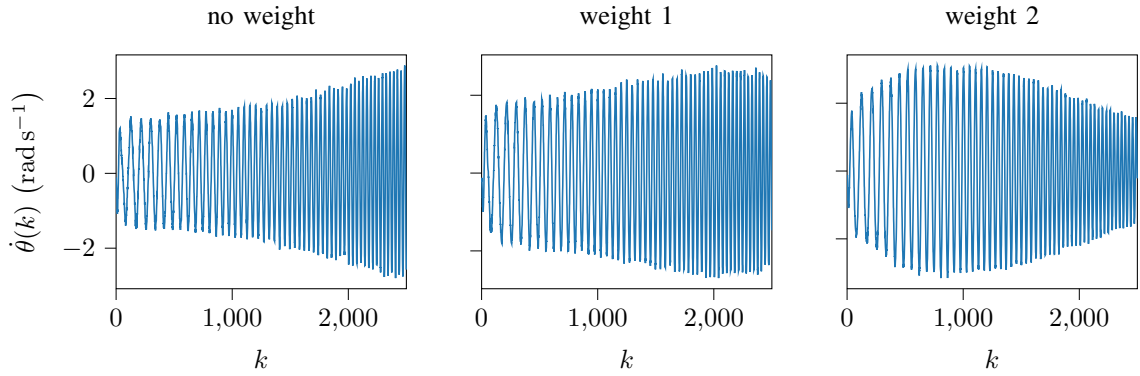}
\caption{Trajectories of context identification experiments.}
\label{fig:trajectories}
\end{figure*}

In case we need to identify the context, we seek to compute the MMD between the current trajectory and all stored trajectories.
As trajectory data is naturally correlated through time, we follow the approach from~\cite{solowjow2020kernel} and subsample the data such that the subsampled trajectories satisfy Assumption~\ref{ass:mixing}.
Therefore, we collect two independent trajectories per context and compute the MMD for increasing values of $a$, which we show in \figref{fig:cont_id_hyperparams}.
We see that for $a>50$, the MMD is reliably kept at a low level.
Thus, during context identification experiments, we subsample by choosing only every 50th sample and then compute the MMD between the current trajectory and all stored trajectories to identify the context.
During our experiments, we identified the context correctly in every iteration in which context identification was required.

\begin{figure*}
\centering 
\begin{tikzpicture}

\definecolor{darkgray176}{RGB}{176,176,176}

\begin{groupplot}[group style={group size=3 by 1}]
\nextgroupplot[
title = {no weight},
ylabel={$\mmdsq$},
xlabel={$a$},
width=0.3\textwidth,
tick align=outside,
tick pos=left,
x grid style={darkgray176},
xmin=0, xmax=100,
xtick style={color=black},
y grid style={darkgray176},
ymin=0, ymax=112.912609033961,
ytick style={color=black}
]
\addplot [semithick, blue]
table {%
0 107.535818127582
1 54.4683233633891
2 35.8139827466916
3 27.5946636689007
4 21.737741087824
5 18.0865063773627
6 15.7313355312895
7 14.2895039299519
8 12.2525507771696
9 10.0182280033595
10 10.1881605081338
11 9.24425366415673
12 8.65574092904742
13 8.25552090578248
14 6.18568618008052
15 3.91444233649071
16 7.25052282904729
17 9.40769816718661
18 5.06847466213029
19 10.5735882215504
20 5.89410530191456
21 5.23980874405192
22 5.04064229546042
23 5.05543160384247
24 4.30947579279384
25 3.93932618565406
26 3.18328140515305
27 4.39515784675892
28 5.38282070121591
29 3.39178979529167
30 1.54753918179843
31 2.44312747018276
32 9.6107510942329
33 3.68475766671169
34 2.45193616010851
35 6.91279247693729
36 4.82537061920741
37 2.57527381169531
38 8.04274290421616
39 7.50400149331673
40 4.70105435493487
41 2.08311981783519
42 1.81936184973127
43 2.22214190566224
44 2.0405274704619
45 3.6740065878312
46 3.77163798785144
47 1.84685191679576
48 1.41937808124708
49 0.997859200887383
50 2.6160152486667
51 5.39571174273991
52 1.98457026811228
53 3.8513015126681
54 2.99836838756545
55 5.3621452984096
56 2.03304477457921
57 2.39370931940753
58 4.40309699249212
59 3.52590471305512
60 4.65907715544549
61 1.31866362911729
62 1.71328714389274
63 1.68284619641686
64 3.4836230481404
65 4.76568494759469
66 2.30864680616762
67 2.09467128101683
68 4.8401081939702
69 2.1514233543185
70 4.52336172390783
71 5.10333213025055
72 1.370737621153
73 2.05792193113349
74 2.16037181965924
75 1.54852159693856
76 1.86750549715044
77 4.45996328659459
78 4.37920593686694
79 4.07210253080857
80 3.53779212308175
81 2.92105673498231
82 0.65889447745733
83 2.26273859446201
84 1.37079821180592
85 0.505925694053984
86 2.14677667936006
87 0.758452415340203
88 1.77249948377161
89 1.62856529086994
90 1.18336749970519
91 1.59293455834911
92 1.43810947350397
93 2.75070788143336
94 3.55906131213012
95 0.682126495284531
96 2.0750788253232
97 0.960487806339807
98 4.49816440322571
99 0
};

\nextgroupplot[
title = {weight~1},
xlabel={$a$},
width=0.3\textwidth,
tick align=outside,
tick pos=left,
x grid style={darkgray176},
xmin=0, xmax=100,
xtick style={color=black},
y grid style={darkgray176},
ymin=0, ymax=57,
ytick style={color=black}
]
\addplot [semithick, blue]
table {%
0 53.8909964277778
1 26.782401307731
2 18.0488753278816
3 13.4373372625633
4 11.3069837700812
5 9.2781805482332
6 8.14512499578141
7 7.85615459059278
8 5.39685136267336
9 6.26102898473841
10 5.52631722839637
11 5.24449637556278
12 5.28543992805606
13 3.54779226819338
14 3.58035600793024
15 5.04636641691324
16 5.6473921691391
17 7.23810044426139
18 3.88287767958572
19 6.95823990703973
20 4.15353641509772
21 2.82722160834486
22 2.52438803967891
23 3.6656285901634
24 2.79623801750633
25 3.87037515097576
26 1.75041053171406
27 1.79028497887039
28 3.00956536126735
29 2.2127888186389
30 0.721083062014105
31 3.89636118122639
32 6.48586225865337
33 4.00760904733482
34 2.24279126365581
35 6.61828684804266
36 2.54003174270794
37 1.90071479150163
38 3.89445596397851
39 6.61257699387053
40 3.00642989283049
41 1.0112828006536
42 1.22095138478806
43 1.18248232355263
44 1.43665281659443
45 1.23072561885453
46 2.6569389858976
47 3.32214371353858
48 1.67213218037523
49 1.67175673782493
50 2.11495811561
51 4.10792475146851
52 0.573233479134476
53 2.71651834400768
54 1.08237996774599
55 2.67047667767259
56 1.40544452630836
57 1.01515530206639
58 2.39221569478435
59 2.72580708134272
60 2.85486397701898
61 2.05087262306693
62 1.70756685266704
63 2.62545756592389
64 4.55184808953517
65 3.25871673920105
66 1.18625957582278
67 2.55213025660241
68 3.95956610164131
69 2.00243513617092
70 1.48347149446874
71 5.15928441124943
72 1.90651605095104
73 0.505887002528722
74 1.35560975309251
75 1.53517872906107
76 0.753018098569435
77 2.08892721953252
78 2.45165755268222
79 4.50539032280954
80 3.11376475245682
81 2.13369217994789
82 0.962842852614625
83 1.25695025668777
84 1.37242875299851
85 0.189802625205227
86 1.75442782594673
87 0.687102883410046
88 0.750577666834302
89 1.96187816483823
90 1.72097249276188
91 0.726943407297701
92 0.914929549505856
93 2.89172662099584
94 1.72003678486282
95 1.76861293526838
96 0.977413925863953
97 0.740095392734355
98 3.22055053686696
99 0
};

\nextgroupplot[
title = {weight~2},
xlabel={$a$},
width=0.3\textwidth,
tick align=outside,
tick pos=left,
x grid style={darkgray176},
xmin=0, xmax=100,
xtick style={color=black},
y grid style={darkgray176},
ymin=0, ymax=30,
ytick style={color=black}
]
\addplot [semithick, blue]
table {%
0 28.0138880357221
1 13.929585360319
2 9.30433663640112
3 6.79418734060666
4 5.28896969817302
5 4.5037199704539
6 3.92225490807945
7 3.24212912682252
8 3.15122035692117
9 2.20348192933784
10 2.4450126783627
11 2.32487913086532
12 1.94049419292778
13 2.18862328935519
14 1.42545698307615
15 0.932883017157149
16 0.802976725096391
17 1.00310252421843
18 0.904478262744613
19 1.36940121341307
20 0.954364070575321
21 0.691602119842544
22 0.791006732726143
23 0.730205826593042
24 1.29795356754099
25 0.455020446504773
26 1.84491651475432
27 1.38258192843802
28 0.98583570448492
29 1.10418385223787
30 2.2136121750637
31 1.22678830097748
32 2.2164899433811
33 2.48566116909161
34 1.7197654214636
35 2.70988742703553
36 4.66772205239255
37 1.91118002856996
38 0.821351336531629
39 1.78915890815738
40 1.82149666446979
41 1.28369043862704
42 0.954154945913106
43 0.46519657555591
44 0.331231141662438
45 0.348367007783824
46 0.383647875635976
47 0.294469035276029
48 0.339115630683432
49 0.435423281830307
50 0.38535108688501
51 0.35132991617898
52 0.467668566747561
53 0.649990966318662
54 0.431350286375226
55 0.806892480552393
56 0.546775864132064
57 1.42772206197669
58 2.21456403537895
59 1.15631294116837
60 0.78782450812044
61 1.97443012357977
62 0.124003997694422
63 0.883777886931131
64 3.22067844134949
65 0.378669101455726
66 0.418524345293519
67 0.343492976700828
68 3.08963989706739
69 2.11658489620154
70 0.486466189587128
71 3.2330715790903
72 2.73947115272216
73 2.16102441147982
74 0.97197012082931
75 0.869590002071329
76 0.227054806806387
77 0.297733180699222
78 0.621327423281686
79 0.35411403080826
80 0.662900130667378
81 0.773252881784629
82 0.316836147485093
83 1.31119528664549
84 0.153791577374074
85 0.845140935847827
86 0.919442827252217
87 0.0595297436799065
88 1.26130962041653
89 0.604330817030919
90 1.67457306386825
91 0.335398737366298
92 0.436300965931996
93 0.816614132223207
94 0.523904318332543
95 0.313597579833801
96 1.17913185956473
97 1.65642739886777
98 1.3507996627075
99 0
};
\end{groupplot}

\end{tikzpicture}
\caption{MMD of context identification experiments for different weights for varying $a$. \capt{For $a>50$, we see that the MMD is at a low level, \ie trajectories are approximately independent.}}
\label{fig:cont_id_hyperparams}
\end{figure*}
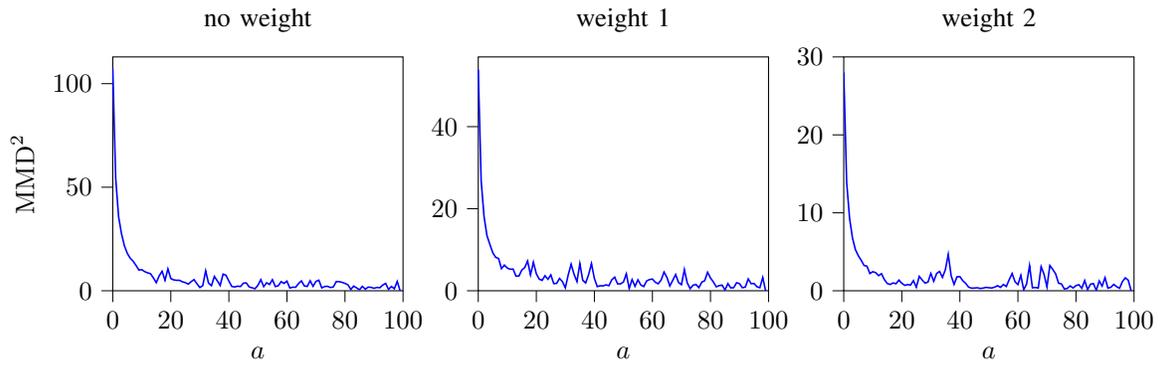

Over time, we build a data set of weight images and contexts that allows us to make more confident classification decisions.
In \figref{fig:camera_results}, we show the classification probability estimates and uncertainty bounds of ten randomly selected images given a data set consisting of \emph{(i)} the first ten images in our data set (top row), \emph{(ii)} the first half of the data set (middle row), and \emph{(iii)} the entire data set (312 images, bottom row).
We can see that the uncertainty steadily decreases. 
After accessing the entire data set, we can make classification decisions with confidence above \SI{70}{\percent} in the case of weight two.
We consider this to be a relatively small data set for image classification.
We further see a few misclassifications, marked by red crosses in \figref{fig:camera_results}, especially for the small data sets.
Those are accompanied by large uncertainty intervals, \ie our algorithm correctly identifies that the output of the classifier should not be trusted in those cases.

\begin{figure*}
\centering
\input{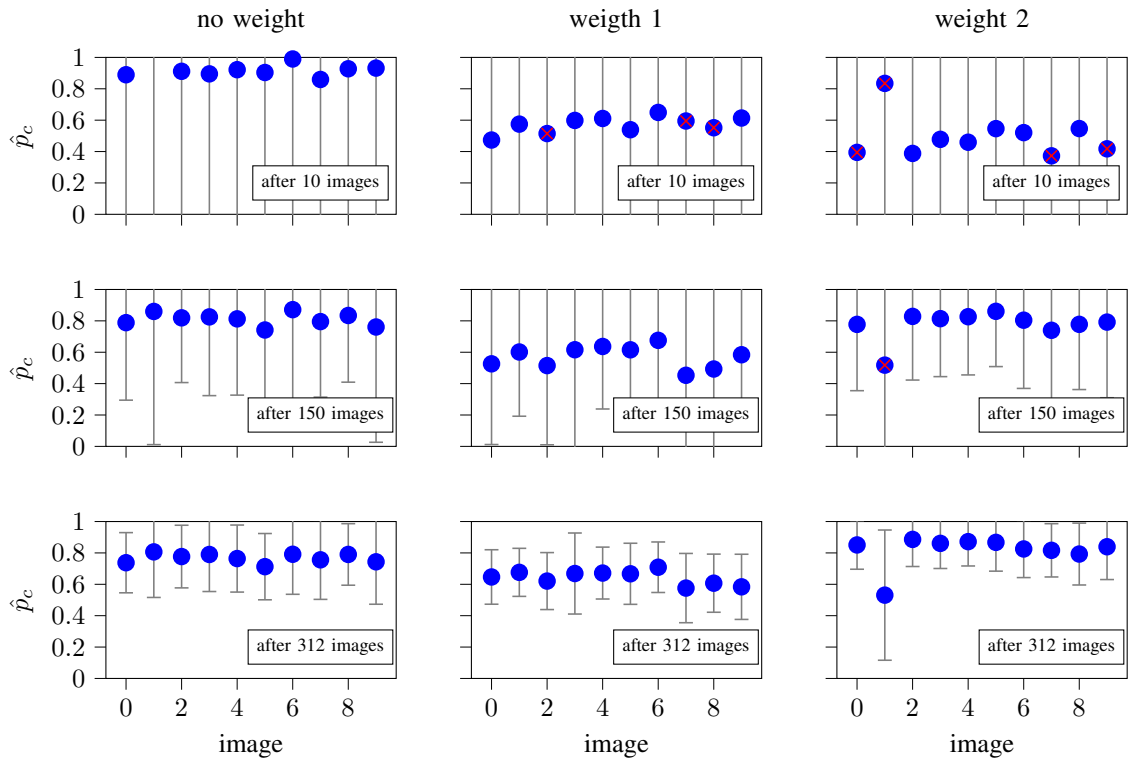}
\caption{Prediction of weights based on camera images. \capt{We show the prediction $\hat{p}_\mathrm{c}$ and uncertainty intervals from Corollary~\ref{cor:prob_uncertainty} for ten images without weight (left), ten with weight one (middle), and ten with weight two (right).
Wrong predictions are marked with red crosses.
From top to bottom, we see how the uncertainty intervals decrease from a data set of ten images, over one with around 150 images (middle), to the full data set of 312 images.}}
\label{fig:camera_results}
\end{figure*}

During all experiments, the pole of the Furuta pendulum never dropped, \ie we successfully retained safety.

\subsection{Comparison}
\label{sec:comp}

Having shown the general applicability of our algorithm, we next compare the resulting algorithm to \textsc{SafeOpt} without our bounds and context identification.
On the one hand, identifying the context in cases where the classifier is too uncertain comes at the expense of requiring additional experimentation.
On the other hand, considering contexts may improve performance and even be necessary to ensure safety.

For the comparison, we again consider the Furuta pendulum, but this time in simulation, again based on the code from~\cite{bleher2022learning}, and change pole mass and length in the simulation code.
We then let \textsc{SafeOpt} optimize the feedback gain multiplied by the pole angle while not providing any information about the current context.
After, we run the algorithm proposed in this work.
Instead of camera images, in the simulated case, we assume that we receive noisy measurements, where the noise variance is normally distributed with a standard deviation of 0.1, of the height of the weight at the beginning of an experiment.
Given the noisy height measurement, we evaluate the classifier and classification bounds and accept the outcome if the lower bound is above $p_\mathrm{safe} = 0.8$.
Otherwise, we perform an identification experiment as before.
The parameter $p_\mathrm{safe}$ is a tuning parameter and mainly depends on the task at hand, \ie the consequences of misclassification and, hence, a potential violation of a safety constraint.
We discuss the choice of $p_\mathrm{safe}$ in more detail in \secref{sec:sensitivity}.
Having certainty about the context, we perform a \textsc{SafeOpt} experiment.
We adopt the kernel parameters for \textsc{SafeOpt}, but reduce the length scales of the kernel for classification to 0.1.

We report results in \figref{fig:comp}.
The left plot shows the experimentation time required for pure \textsc{SafeOpt}, \textsc{SafeOpt} with both context identification and classification, and \textsc{SafeOpt} where we run a context identification at the beginning of each experiment without attempting any classification.
Each context identification and also each \textsc{SafeOpt} experiment lasts \num{2500} samples with an underlying sampling frequency of \SI{200}{\hertz}, and we run the entire loop for \num{1000} iterations.
Clearly, when identifying the context at each time step, the overall required experimentation time is double the time the pure \textsc{SafeOpt} algorithm needs.
Nevertheless, we see that when leveraging the classification, we can already, at this stage, save some time. 
Considering that it initially requires some training time until the classification starts to be effective, we can expect even more significant relative savings when running the algorithm for a longer time.
But, certainly, running \textsc{SafeOpt} while ignoring the unknown contexts will require the least time.
However, in the right plot of \figref{fig:comp} we also see the downside of this approach: while both our scheme and the one that identifies the context at the beginning of each experiment have zero failures after \num{1000} iterations, \textsc{SafeOpt} without considering contexts accumulated 55.
While this is uncritical in a simulation experiment, in real experiments with costly and fragile hardware, this can be problematic.
The other extreme case would be to consider \textsc{SafeOpt} with contexts, assuming contexts to be known.
In essence, that would result in equivalent performance in terms of failures as the runs we did with context identification and the runs we did with both context identification \emph{and} classification, as both always recovered the true context.
However, it would also result in the same training time as \textsc{SafeOpt}, as we would assume that the context is known, rendering both context identification and classification unnecessary.
Thus, under this assumption, the contributions of this paper would not be required, and pure \textsc{SafeOpt}, or a more advanced version of the algorithm, would be the more suitable choice.

\begin{figure}
\centering
\begin{tikzpicture}
\begin{groupplot}[group style={group size=2 by 1, horizontal sep=1.5cm}]
\nextgroupplot[
tick align=outside,
tick pos=left,
width=0.22\textwidth,
height=0.15\textheight,
ylabel={time (h)},
symbolic x coords={\textsc{SafeOpt}, Ours, ContID},
ymin=0, ymax=7,
xtick=data,
]
\addplot[ybar, fill=blue]coordinates{
    (\textsc{SafeOpt}, 3.472222222)
    (Ours, 6.437455556)
    (ContID, 6.9444)
};

\nextgroupplot[
tick align=outside,
tick pos=left,
width=0.22\textwidth,
height=0.15\textheight,
ylabel={\# failures},
symbolic x coords={\textsc{SafeOpt}, Ours, ContID},
ymin=0, ymax=60,
xtick=data,
]
\addplot[ybar, fill=blue]coordinates{
    (\textsc{SafeOpt}, 55)
    (Ours, 0)
    (ContID, 0)
};
\end{groupplot}
\end{tikzpicture}
\caption{Training time and number of failures for pure \textsc{SafeOpt}, \textsc{SafeOpt} with context identification \emph{and} classification (our), and \textsc{SafeOpt} with context identification before the start of every experiment (ContID). \capt{While our extensions require more samples for the additional context identification and, therefore, more training time, we do not incur any failures while we have several when only using \textsc{SafeOpt}.}}
\label{fig:comp}
\end{figure}
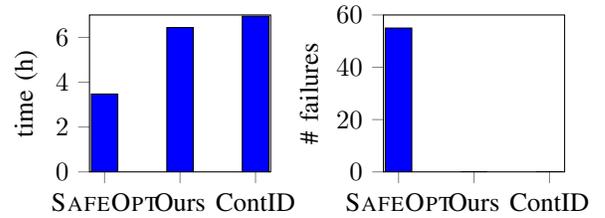

\subsection{Sensitivity analysis and limitations}
\label{sec:sensitivity}

The two examples have shown the benefit of the classification bounds proposed in this paper.
Nevertheless, when to accept a classified context without context identification depends on $p_\mathrm{safe}$, which is a tuning parameter that we set to 0.8 in the previous examples.
In this section, we show how the choice of this parameter influences the classification results.
Further, we discuss a case with more and gradually changing contexts.

We still consider the same setup as before, with the simulated Furuta pendulum and noisy height measurements providing information about the current contexts.
This time, we consider five contexts, where the heights are $h \in \{1, 2, 2.5, 2.75, 2.875\}$, disturbed with normally distributed noise with a standard deviation of 0.1.
Thus, in this case, at least the last two contexts are hard to distinguish for the classifier.
We then run five times \num{2000} experiments picking one of $p_\mathrm{safe}\in\{0.5, 0.6, 0.7, 0.8, 0.9\}$ for each of the five runs, and report in \figref{fig:sensitivity} which contexts were classified correctly (in blue) and incorrectly (in red) for the different $p_\mathrm{safe}$ values.
Clearly, as we increase the safety threshold, the classifier less often exceeds it. On the other hand, also the number of misclassifications decreases as we increase $p_\mathrm{safe}$.
We can further see that for the contexts that are much farther apart from each other than the noise level, there are no classification errors even for $p_\mathrm{safe}=0.5$.
However, especially for $h=2.75$, we have relatively many misclassifications and, even for $p_\mathrm{safe}=0.5$, overall only few cases in which the classifier exceeds the threshold.
Thus, when contexts are hard to distinguish, our classifier will often report low certainty and render the context classification necessary.

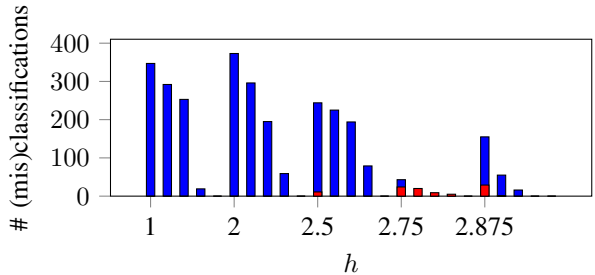
\begin{figure}
\centering 
\begin{tikzpicture}
\begin{groupplot}[group style={group size=1 by 1}]

\nextgroupplot[
ylabel={\# (mis)classifications},
tick align=outside,
tick pos=left,
width=0.44\textwidth,
height=0.15\textheight,
ymin=0,
xlabel={$h$},
xtick={0, 1, 2, 3, 4},
xticklabels={1, 2, 2.5, 2.75, 2.875},
bar width=0.1,
]
\addplot[ybar, fill=blue]coordinates{
    (0, 347)
    (0.2, 292)
    (0.4, 253)
    (0.6, 19)
    (0.8, 0)
    (1, 373)
    (1.2, 296)
    (1.4, 195)
    (1.6, 59)
    (1.8, 0)
    (2, 244)
    (2.2, 225)
    (2.4, 194)
    (2.6, 79)
    (2.8, 0)
    (3, 43)
    (3.2, 0)
    (3.4, 0)
    (3.6, 0)
    (3.8, 0)
    (4, 155)
    (4.2, 55)
    (4.4, 16)
    (4.6, 0)
    (4.8, 0)
};
\addplot[ybar, fill=red]coordinates{
    (0, 0)
    (0.2, 0)
    (0.4, 0)
    (0.6, 0)
    (0.8, 0)
    (1, 0)
    (1.2, 0)
    (1.4, 0)
    (1.6, 0)
    (1.8, 0)
    (2, 11)
    (2.2, 0)
    (2.4, 0)
    (2.6, 0)
    (2.8, 0)
    (3, 24)
    (3.2, 20)
    (3.4, 9)
    (3.6, 5)
    (3.8, 0)
    (4, 29)
    (4.2, 1)
    (4.4, 0)
    (4.6, 0)
    (4.8, 0)
};
\end{groupplot}
\end{tikzpicture}
\caption{Correct and incorrect classifications for the simulated Furuta pendulum with noisy measurements and different values of $p_\mathrm{safe}$. \capt{For each context, we see the number of correct classifications in blue and misclassifications in red, from left to right, for increasing $p_\mathrm{safe}$. The lower $p_\mathrm{safe}$, the more contexts we can classify, but the larger also our errors, especially for the two contexts that are close together.}}
\label{fig:sensitivity}
\end{figure}

\subsection{Classification bounds}
\label{sec:class_eval}

Next, we evaluate the performance of the classification bounds.
We first qualitatively compare the nature of our uncertainty bounds with more standard, expected risk bounds from~\cite{hsu2018hyperparameter} and the recently proposed deterministic bounds from~\cite{maddalena2021deterministic} in a simple, synthetic example.
Then, we demonstrate the applicability of our bounds in two standard classification benchmarks: the modified National Institute of Standards and Technology (MNIST) dataset~\cite{lecun1998gradient} and the German traffic sign recognition benchmark (GTSRB)~\cite{stallkamp2012man}.
In this part of the evaluation, we disregard the uncertainty from context identification but investigate directly the error $\abs{p_\mathrm{c}(y)-\check{p}_\mathrm{c}(y)}$.

\fakepar{Qualitative comparison}
For a qualitative comparison, we choose a probability function $p_0(y) = (1+\exp(-y+1))^{-1}$ and $p_1(y) = 1-p_0(y)$, where $y$ is a scalar parameter. 
The training set consists of 50 $y$ values in the range $-6$ to $-4.7$, 50 $y$ values between 0.5 and 1.78, and 50 $x$ values between 5.7 and 7.
We sample a context for each $x$ value. 
The context is either zero or one with probability $p_0(y)$ and $p_1(y)$, respectively.

For all three approaches, we then compute the estimate $\hat{p}_c$ using CMEs as presented in \secref{sec:prel_class}, and their respective bounds for 100 $y$ values sampled uniformly in the interval $[-6, 7]$.
For all approaches, we use a Gaussian kernel with length scale 1.
In the appendix of~\cite{scharnhorst2021robust}, the authors present a method to empirically estimate a bound on the RKHS norm of a given function.
For the function $p_0$, we get $\Gamma=2$ as a conservative estimate in the region where we evaluate the function.
Thus, we choose $\Gamma=2$.

The bounds in~\cite[Thm.~4]{hsu2018hyperparameter} are expected risk bounds.
As such, they do not directly allow us to bound the error $\abs{p_c(y)-\check{p}_c(y)}$.
Instead, we can infer the likelihood of misclassifying a sample for any input $y$.
The bounds are extremely conservative if used in such a way and report a misclassification risk above \SI{99}{\percent} for the chosen hyperparameters.
However, making such predictions is not the purpose for which these bounds were developed.
In~\cite{hsu2018hyperparameter}, they were used to tune the hyperparameters of a CME-based classifier.
The authors showed that the bounds are very useful in providing a trade-off between accuracy on the training data and model complexity that leads to good generalization properties.
Nevertheless, even with the optimized hyperparameters in~\cite{hsu2018hyperparameter}, the bounds are too conservative to apply to our problem setting.

The bounds from~\cite{maddalena2021deterministic} are closer to the ones proposed in this article.
Similar to ours, they are input-dependent and, thus, give, for any specific input $x$, an upper bound on the deviation $\abs{p_c(y)-\check{p}_c(y)}$.
However, while our bound is a high probability bound, the bound from~\cite{maddalena2021deterministic} is deterministic.
Their central assumption is that training data is corrupted by bounded noise with a known bound.
In our setting, we aim at estimating $p_c(y)$ while only receiving binary labels.
We can interpret this as measuring $p_c(y)$ with a maximum error of one.
Then, we can use the bounds from~\cite{maddalena2021deterministic}.
In particular, we use the simplified version of the bounds provided in~\cite[Thm.~1]{maddalena2021deterministic} that do not require computing the RKHS norm of $\bar{p}_c$.

In \figref{fig:comparison}, we compare our bounds (bottom) with those from~\cite{maddalena2021deterministic} (top).
We show the true probability function $p_0(y)$, the estimates $\check{p}_0(y)$, and the bounds in both plots.
The data-dependent nature can be seen in both figures. 
However, the bounds from~\cite{maddalena2021deterministic} are way more conservative than ours.
This is natural since the bounds from~\cite{maddalena2021deterministic} are deterministic, \ie they need to hold for any ground truth function compatible with the data and noise model, while ours are high probability statements.
Thus, while the bounds from~\cite{maddalena2021deterministic} are essential results for general function estimation and applications in, for instance, robust control, the bounds are non-informative for classification. 

\begin{figure}
\centering
\input{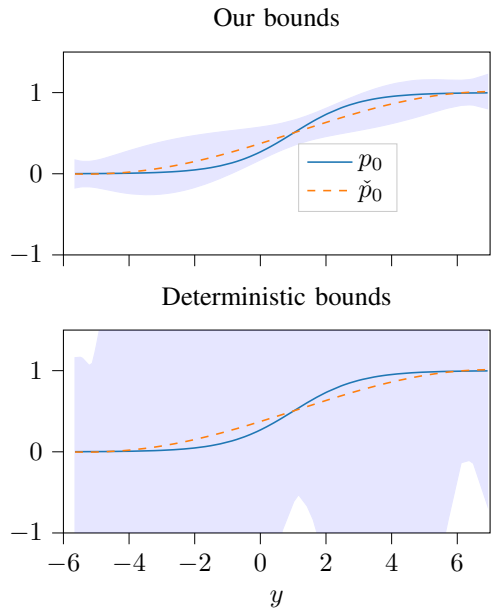}
\caption{Our bounds compared with those of~\cite{maddalena2021deterministic}. \capt{The bounds are illustrated through the blue shaded areas. Compared to our bounds (top), the bounds from~\cite{maddalena2021deterministic} (bottom) are, due to their deterministic nature, way more conservative and, therefore, non-informative for classification.}}
\label{fig:comparison}
\end{figure}

This comparison shows that state-of-the-art bounds for multi-class classification do not meet the requirements for our problem setting.
We require bounds that are both input-dependent and probabilistic such that they yield practically usable worst-case bounds.
In contrast to existing bounds, the bounds we derived in \secref{sec:class} meet both requirements.

\fakepar{MNIST}
To challenge the scalability of the bounds, we next consider the popular MNIST dataset.
The MNIST dataset consists of images of handwritten digits.
Thus, the task for our classifier is to predict which digit can be seen in a specific image.
At the same time, we seek to infer how certain we are about the classification through our bounds\footnote{The code for this example is available at \href{https://github.com/baumanndominik/cme_based_classification_bounds}{https://github.com/baumanndominik/cme\_based\_classification\_bounds}.}.

For image classification, we again choose a Gaussian kernel with a log length scale of $7.5$ and log magnitude of $2.6$.
We normalize the images such that the pixels take values in $[-1, 1]$.  

The MNIST dataset is split into a training and test set.
We use the first \num{10000} training images to train our algorithm and then evaluate it on the first occurrence of each digit in the test dataset.
For the test examples, we show both the probability of the most likely digit and the uncertainty bounds in \figref{fig:mnist}.
After having seen \num{10000} training examples, the uncertainty in many test cases is still very high and would not enable us to make confident predictions.
While this represents a limitation, it is a natural one.
A single observed outcome for a specific parameter setting can already yield significant insights in a regression task. 
This is not the case in classification.
If we throw a coin once and it comes up tails, we cannot infer whether or not the coin is likely to be fair.
Especially in this light, the bounds developed herein are an essential asset in the classification setting.
Popular classifiers have been reported to be over-confident~\cite{bai2021don}.
Hence, it is crucial to add reliable bounds, particularly if the classifier's output is used in safety-critical environments.

However, \figref{fig:mnist} also shows that in some instances, \eg for zero, one, and seven, we are confident that our estimator classifies correctly.
For comparison, we computed the expected risk bounds from~\cite[Thm.~4]{hsu2018hyperparameter}.
Those reveal a significant misclassification risk over the entire input domain.
Thus, they would not let us make any confident classification decision, rendering the entire classification useless.
This underlines the importance of input-dependent uncertainty bounds for safe learning.
If we either accept all or none of the predictions, it may take us too long until we are confident enough.
With the input-dependent bounds developed in this article, we judge the prediction uncertainty locally at the current input and can make confident classification decisions in specific parts of the input space even if, overall, the uncertainty is still high.

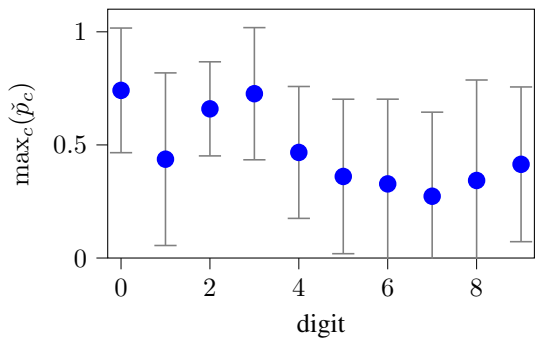
\begin{figure}
\centering
\tikzsetnextfilename{mnist}
\begin{tikzpicture}

\definecolor{darkgray176}{RGB}{176,176,176}
\definecolor{steelblue31119180}{RGB}{31,119,180}

\begin{groupplot}[group style={group size=1 by 1}]
\nextgroupplot[
tick align=outside,
tick pos=left,
x grid style={darkgray176},
xmin=-0.3, xmax=9.3,
xtick style={color=black},
y grid style={darkgray176},
ymin=0, ymax=1.1,
ytick style={color=black},
xlabel=digit,
ylabel=$\max_c(\check{p}_c)$,
height=0.2\textheight,
width=0.4\textwidth
]
\addplot [semithick, gray]
table {%
0 0.465677158892848
0 1.01663590302034
};
\addplot [semithick, gray]
table {%
-0.25 1.01663590302034
0.25 1.01663590302034
};
\addplot [semithick, gray]
table {%
-0.25 0.465677158892848
0.25 0.465677158892848
};
\addplot [semithick, blue, mark=*, mark size=3, mark options={solid}, only marks]
table {%
0 0.741156530956596
};
\addplot [semithick, gray]
table {%
1 0.055398260719249
1 0.818408280263734
};
\addplot [semithick, gray]
table {%
0.75 0.818408280263734
1.25 0.818408280263734
};
\addplot [semithick, gray]
table {%
0.75 0.055398260719249
1.25 0.055398260719249
};
\addplot [semithick, blue, mark=*, mark size=3, mark options={solid}, only marks]
table {%
1 0.436903270491491
};
\addplot [semithick, gray]
table {%
2 0.451481898156429
2 0.867427083659271
};
\addplot [semithick, gray]
table {%
1.75 0.867427083659271
2.25 0.867427083659271
};
\addplot [semithick, gray]
table {%
1.75 0.451481898156429
2.25 0.451481898156429
};
\addplot [semithick, blue, mark=*, mark size=3, mark options={solid}, only marks]
table {%
2 0.65945449090785
};
\addplot [semithick, gray]
table {%
3 0.434266121441912
3 1.01853764635961
};
\addplot [semithick, gray]
table {%
2.75 1.01853764635961
3.25 1.01853764635961
};
\addplot [semithick, gray]
table {%
2.75 0.434266121441912
3.25 0.434266121441912
};
\addplot [semithick, blue, mark=*, mark size=3, mark options={solid}, only marks]
table {%
3 0.726401883900759
};
\addplot [semithick, gray]
table {%
4 0.175242073873674
4 0.758311724320307
};
\addplot [semithick, gray]
table {%
3.75 0.758311724320307
4.25 0.758311724320307
};
\addplot [semithick, gray]
table {%
3.75 0.175242073873674
4.25 0.175242073873674
};
\addplot [semithick, blue, mark=*, mark size=3, mark options={solid}, only marks]
table {%
4 0.46677689909699
};
\addplot [semithick, gray]
table {%
5 0.0191327160756961
5 0.702042685123261
};
\addplot [semithick, gray]
table {%
4.75 0.702042685123261
5.25 0.702042685123261
};
\addplot [semithick, gray]
table {%
4.75 0.0191327160756961
5.25 0.0191327160756961
};
\addplot [semithick, blue, mark=*, mark size=3, mark options={solid}, only marks]
table {%
5 0.360587700599478
};
\addplot [semithick, gray]
table {%
6 -0.0468152062768332
6 0.702264260169678
};
\addplot [semithick, gray]
table {%
5.75 0.702264260169678
6.25 0.702264260169678
};
\addplot [semithick, gray]
table {%
5.75 -0.0468152062768332
6.25 -0.0468152062768332
};
\addplot [semithick, blue, mark=*, mark size=3, mark options={solid}, only marks]
table {%
6 0.327724526946423
};
\addplot [semithick, gray]
table {%
7 -0.0984640546248255
7 0.644799238133161
};
\addplot [semithick, gray]
table {%
6.75 0.644799238133161
7.25 0.644799238133161
};
\addplot [semithick, gray]
table {%
6.75 -0.0984640546248255
7.25 -0.0984640546248255
};
\addplot [semithick, blue, mark=*, mark size=3, mark options={solid}, only marks]
table {%
7 0.273167591754168
};
\addplot [semithick, gray]
table {%
8 -0.101371409194506
8 0.786979402898361
};
\addplot [semithick, gray]
table {%
7.75 0.786979402898361
8.25 0.786979402898361
};
\addplot [semithick, gray]
table {%
7.75 -0.101371409194506
8.25 -0.101371409194506
};
\addplot [semithick, blue, mark=*, mark size=3, mark options={solid}, only marks]
table {%
8 0.342803996851928
};
\addplot [semithick, gray]
table {%
9 0.0722629639239631
9 0.756271548392851
};
\addplot [semithick, gray]
table {%
8.75 0.756271548392851
9.25 0.756271548392851
};
\addplot [semithick, gray]
table {%
8.75 0.0722629639239631
9.25 0.0722629639239631
};
\addplot [semithick, blue, mark=*, mark size=3, mark options={solid}, only marks]
table {%
9 0.414267256158407
};
\end{groupplot}

\end{tikzpicture}
\caption{Our classifier and bounds applied to the MNIST dataset. \capt{After having seen \num{10000} training images, we can already make confident decisions for some digits.}}
\label{fig:mnist}
\end{figure}

\fakepar{German traffic sign recognition benchmark}
While the MNIST dataset is widely used, it may not be obvious how misclassifying a digit could be fatal.
Therefore, we next consider the German traffic sign recognition benchmark (GTSRB)~\cite{stallkamp2012man}.
The GTSRB contains images of different traffic signs that should be classified.
Suppose an algorithm that classifies traffic signs is used, for instance, within a self-driving car that chooses its driving policy based on this classification. In that case, we must be sure about our predictions.

Also here, we consider a Gaussian kernel, this time with a log length scale of $7$ and log magnitude of $1.5$.
Besides normalizing the pixel values, we rescale the images, which a priori have varying sizes, to $32\times 32$ pixels and convert them to grayscale images.

Similar to MNIST, the GTSRB is divided into training and test set.
We randomly select ten traffic signs.
Then, we train the classifier on the first \num{1000} images that contain those signs and use the first occurrence of those signs in the test set for evaluation.
As shown in \figref{fig:gtsrb} (top), the uncertainty is still very high.
In particular, since the lower bound of the classification probability barely reaches \SI{50}{\percent}, at this stage, the classifier should, if at all, only be trusted to recognize sign number nine.
We then re-evaluate the same signs after providing \num{10000} images for training.
This significantly reduces the uncertainty (see \figref{fig:gtsrb}, bottom).
This shows that the bounds can easily be used in online learning settings and tightened as we receive more data.
After having seen \num{10000} training examples, we can also confidently (and correctly) identify signs three and eight.
Meanwhile, for sign two, for instance, the bounds tell us that the algorithm cannot provide a reliable classification probability.
Since misclassification in this example may cause accidents, knowing for which signs we cannot rely on the classifier is precious information.
This shows that the classification bounds are informative in cases with relatively few training data, as they can clearly indicate when we can start to trust the classifier.

\begin{figure}
\centering
\tikzsetnextfilename{gtrsb_low_training}
\begin{tikzpicture}

\definecolor{darkgray176}{RGB}{176,176,176}
\definecolor{steelblue31119180}{RGB}{31,119,180}

\begin{axis}[
title={after \num{1000} training points},
tick align=outside,
tick pos=left,
x grid style={darkgray176},
xmin=-0.3, xmax=9.3,
xtick style={color=black},
y grid style={darkgray176},
ymin=0, ymax=1.1,
ytick style={color=black},
xticklabels=\empty,
ylabel=$\max_c(\check{p}_c)$,
height=0.2\textheight,
width=0.4\textwidth
]
\addplot [semithick, gray]
table {%
0 0.233050627369123
0 0.729346334963431
};
\addplot [semithick, gray]
table {%
-0.25 0.729346334963431
0.25 0.729346334963431
};
\addplot [semithick, gray]
table {%
-0.25 0.233050627369123
0.25 0.233050627369123
};
\addplot [semithick, blue, mark=*, mark size=3, mark options={solid}, only marks]
table {%
0 0.481198481166277
};
\addplot [semithick, gray]
table {%
1 0.20841473036626
1 0.796084307332776
};
\addplot [semithick, gray]
table {%
0.75 0.796084307332776
1.25 0.796084307332776
};
\addplot [semithick, gray]
table {%
0.75 0.20841473036626
1.25 0.20841473036626
};
\addplot [semithick, blue, mark=*, mark size=3, mark options={solid}, only marks]
table {%
1 0.502249518849518
};
\addplot [semithick, gray]
table {%
2 0.241254267096399
2 0.633204733683757
};
\addplot [semithick, gray]
table {%
1.75 0.633204733683757
2.25 0.633204733683757
};
\addplot [semithick, gray]
table {%
1.75 0.241254267096399
2.25 0.241254267096399
};
\addplot [semithick, blue, mark=*, mark size=3, mark options={solid}, only marks]
table {%
2 0.437229500390078
};
\addplot [semithick, gray]
table {%
3 0.399257423250944
3 0.960158061801454
};
\addplot [semithick, gray]
table {%
2.75 0.960158061801454
3.25 0.960158061801454
};
\addplot [semithick, gray]
table {%
2.75 0.399257423250944
3.25 0.399257423250944
};
\addplot [semithick, blue, mark=*, mark size=3, mark options={solid}, only marks]
table {%
3 0.679707742526199
};
\addplot [semithick, gray]
table {%
4 0.33034328741834
4 0.785560087376028
};
\addplot [semithick, gray]
table {%
3.75 0.785560087376028
4.25 0.785560087376028
};
\addplot [semithick, gray]
table {%
3.75 0.33034328741834
4.25 0.33034328741834
};
\addplot [semithick, blue, mark=*, mark size=3, mark options={solid}, only marks]
table {%
4 0.557951687397184
};
\addplot [semithick, gray]
table {%
5 0.029722539714524
5 0.738792003105048
};
\addplot [semithick, gray]
table {%
4.75 0.738792003105048
5.25 0.738792003105048
};
\addplot [semithick, gray]
table {%
4.75 0.029722539714524
5.25 0.029722539714524
};
\addplot [semithick, blue, mark=*, mark size=3, mark options={solid}, only marks]
table {%
5 0.384257271409786
};
\addplot [semithick, gray]
table {%
6 0.136088858822705
6 0.593039802682899
};
\addplot [semithick, gray]
table {%
5.75 0.593039802682899
6.25 0.593039802682899
};
\addplot [semithick, gray]
table {%
5.75 0.136088858822705
6.25 0.136088858822705
};
\addplot [semithick, blue, mark=*, mark size=3, mark options={solid}, only marks]
table {%
6 0.364564330752802
};
\addplot [semithick, gray]
table {%
7 -0.0347744354521784
7 0.753902815274272
};
\addplot [semithick, gray]
table {%
6.75 0.753902815274272
7.25 0.753902815274272
};
\addplot [semithick, gray]
table {%
6.75 -0.0347744354521784
7.25 -0.0347744354521784
};
\addplot [semithick, blue, mark=*, mark size=3, mark options={solid}, only marks]
table {%
7 0.359564189911047
};
\addplot [semithick, gray]
table {%
8 0.313143030678789
8 0.855344743121016
};
\addplot [semithick, gray]
table {%
7.75 0.855344743121016
8.25 0.855344743121016
};
\addplot [semithick, gray]
table {%
7.75 0.313143030678789
8.25 0.313143030678789
};
\addplot [semithick, blue, mark=*, mark size=3, mark options={solid}, only marks]
table {%
8 0.584243886899902
};
\addplot [semithick, gray]
table {%
9 0.561891641720247
9 1.14201036853352
};
\addplot [semithick, gray]
table {%
8.75 1.14201036853352
9.25 1.14201036853352
};
\addplot [semithick, gray]
table {%
8.75 0.561891641720247
9.25 0.561891641720247
};
\addplot [semithick, blue, mark=*, mark size=3, mark options={solid}, only marks]
table {%
9 0.851951005126883
};
\end{axis}

\end{tikzpicture}
\hspace{0.5cm}
\tikzsetnextfilename{GTSRB}
\begin{tikzpicture}

\definecolor{darkgray176}{RGB}{176,176,176}
\definecolor{steelblue31119180}{RGB}{31,119,180}

\begin{axis}[
title={after \num{10000} training points},
tick align=outside,
tick pos=left,
x grid style={darkgray176},
xmin=-0.3, xmax=9.3,
xtick style={color=black},
y grid style={darkgray176},
ymin=0, ymax=1.1,
ylabel=$\max_c(\check{p}_c)$,
xlabel=sign,
height=0.2\textheight,
width=0.4\textwidth
]
\addplot [semithick, gray]
table {%
0 0.321732340562627
0 0.63198775796542
};
\addplot [semithick, gray]
table {%
-0.25 0.63198775796542
0.25 0.63198775796542
};
\addplot [semithick, gray]
table {%
-0.25 0.321732340562627
0.25 0.321732340562627
};
\addplot [semithick, blue, mark=*, mark size=3, mark options={solid}, only marks]
table {%
0 0.476860049264023
};
\addplot [semithick, gray]
table {%
1 0.260835041699509
1 0.63668729648022
};
\addplot [semithick, gray]
table {%
0.75 0.63668729648022
1.25 0.63668729648022
};
\addplot [semithick, gray]
table {%
0.75 0.260835041699509
1.25 0.260835041699509
};
\addplot [semithick, blue, mark=*, mark size=3, mark options={solid}, only marks]
table {%
1 0.448761169089865
};
\addplot [semithick, gray]
table {%
2 0.29209110774204
2 0.537982595081233
};
\addplot [semithick, gray]
table {%
1.75 0.537982595081233
2.25 0.537982595081233
};
\addplot [semithick, gray]
table {%
1.75 0.29209110774204
2.25 0.29209110774204
};
\addplot [semithick, blue, mark=*, mark size=3, mark options={solid}, only marks]
table {%
2 0.415036851411636
};
\addplot [semithick, gray]
table {%
3 0.531777365396947
3 0.877431748425139
};
\addplot [semithick, gray]
table {%
2.75 0.877431748425139
3.25 0.877431748425139
};
\addplot [semithick, gray]
table {%
2.75 0.531777365396947
3.25 0.531777365396947
};
\addplot [semithick, blue, mark=*, mark size=3, mark options={solid}, only marks]
table {%
3 0.704604556911043
};
\addplot [semithick, gray]
table {%
4 0.42775857162001
4 0.717036287021434
};
\addplot [semithick, gray]
table {%
3.75 0.717036287021434
4.25 0.717036287021434
};
\addplot [semithick, gray]
table {%
3.75 0.42775857162001
4.25 0.42775857162001
};
\addplot [semithick, blue, mark=*, mark size=3, mark options={solid}, only marks]
table {%
4 0.572397429320722
};
\addplot [semithick, gray]
table {%
5 0.192400069689057
5 0.63383036106086
};
\addplot [semithick, gray]
table {%
4.75 0.63383036106086
5.25 0.63383036106086
};
\addplot [semithick, gray]
table {%
4.75 0.192400069689057
5.25 0.192400069689057
};
\addplot [semithick, blue, mark=*, mark size=3, mark options={solid}, only marks]
table {%
5 0.413115215374958
};
\addplot [semithick, gray]
table {%
6 0.213804702708768
6 0.501278706224531
};
\addplot [semithick, gray]
table {%
5.75 0.501278706224531
6.25 0.501278706224531
};
\addplot [semithick, gray]
table {%
5.75 0.213804702708768
6.25 0.213804702708768
};
\addplot [semithick, blue, mark=*, mark size=3, mark options={solid}, only marks]
table {%
6 0.357541704466649
};
\addplot [semithick, gray]
table {%
7 0.130467046197966
7 0.626697436258587
};
\addplot [semithick, gray]
table {%
6.75 0.626697436258587
7.25 0.626697436258587
};
\addplot [semithick, gray]
table {%
6.75 0.130467046197966
7.25 0.130467046197966
};
\addplot [semithick, blue, mark=*, mark size=3, mark options={solid}, only marks]
table {%
7 0.378582241228276
};
\addplot [semithick, gray]
table {%
8 0.443540867075796
8 0.776313366433357
};
\addplot [semithick, gray]
table {%
7.75 0.776313366433357
8.25 0.776313366433357
};
\addplot [semithick, gray]
table {%
7.75 0.443540867075796
8.25 0.443540867075796
};
\addplot [semithick, blue, mark=*, mark size=3, mark options={solid}, only marks]
table {%
8 0.609927116754577
};
\addplot [semithick, gray]
table {%
9 0.692773605525063
9 1.05767052538077
};
\addplot [semithick, gray]
table {%
8.75 1.05767052538077
9.25 1.05767052538077
};
\addplot [semithick, gray]
table {%
8.75 0.692773605525063
9.25 0.692773605525063
};
\addplot [semithick, blue, mark=*, mark size=3, mark options={solid}, only marks]
table {%
9 0.875222065452918
};
\end{axis}

\end{tikzpicture}
\caption{Results of the GTSRB dataset. \capt{Also in this case, we can return practically useful bounds on classification probabilities.}}
\label{fig:gtsrb}
\end{figure}
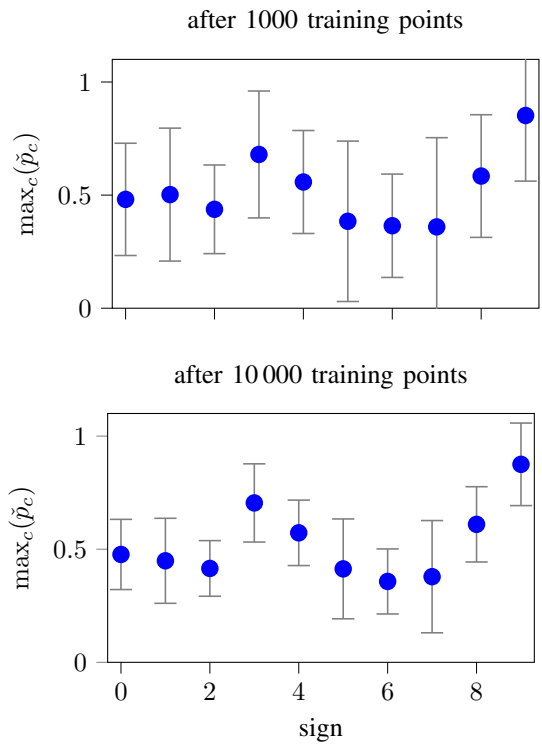




\section{Conclusions}
\label{sec:conclusions}

This paper presents a safe learning algorithm for context-conditional dynamical systems with unknown contexts.
We show how contexts can be identified from data or classified using conditional mean embeddings, providing high-probability guarantees in both cases.
Subsequently, we show how they can be combined with a popular safe learning algorithm.
We demonstrate that given measurements that allow us to distinguish contexts clearly, the classification bounds can save training time.
Otherwise, if the measurements do not allow us to distinguish contexts clearly, we need to identify the contexts through experiments, which increases training time.

A further contribution of this work is the derivation of frequentist uncertainty intervals for a multi-class classifier.
This result has applications beyond the area of safe learning.
Clearly, a limiting factor of the bounds is that we require an upper bound on the RKHS norm.
The bounds share this problem with \textsc{SafeOpt} and related safe learning approaches, which also require an upper bound on this norm.
Thus, estimating this norm from data is subject to ongoing research.

\section*{Acknowledgments}
The authors would like to thank Alexander von Rohr, Sebastian Mair, Christian Fiedler, Friedrich Solowjow, Antonio Ribeiro, and Torbj\"{o}rn Wigren for insightful discussions and comments on earlier manuscript versions, and Claas Thesing for providing the Python implementation of the MMD test.

{\appendix
\subsection{Proofs}
Here, we present extended proofs for Lemmas~\ref{lem:est_uncertainty} and~\ref{lem:meas_uncertainty}.
Before stating the proofs, we introduce some useful derivations used throughout.

\subsubsection{Useful derivations}
\label{sec:derivations}

In this section, we restate and adapt some derivations from the proof of Theorem~2 in~\cite{chowdhury2017kernelized}, which we will use in the proofs of our main results.
We first define the feature map $\varphi(y)\coloneqq k(y,\cdot)$ that maps any point from $\mathcal{Y}$ to the RKHS $\mathcal{H}_k$.
Since we define the inputs $y$ to take values in the real numbers, we can further define the inner product in the RKHS as $\langle g,h\rangle_k=g^\transp h$ and the RKHS norm $\norm{g}_k=\sqrt{g^\transp g}$ for any two functions $f$ and $g$ in $\mathcal{H}_k$.
Defining $\Phi\coloneqq\begin{pmatrix}\varphi(y_1)^\transp,\ldots,\varphi(y_n)^\transp\end{pmatrix}^\transp$, we can write $K=\Phi\Phi^\transp$, $K_y=\Phi\varphi(y)$, and $\begin{pmatrix}p_c(y_1),\ldots,p_c(y_n)\end{pmatrix}=\Phi p_c$.

The matrix $(\Phi^\transp\Phi+n\lambda I)$ is strictly positive definite.
Thus, we have
\begin{equation}
\label{eqn:shift_phi}
\Phi^\transp(\Phi\Phi^\transp+n\lambda I)^{-1} = (\Phi^\transp\Phi+n\lambda I)^{-1}\Phi^\transp. 
\end{equation}
Using~\eqref{eqn:shift_phi}, we can further conclude
\begin{equation}
\label{eqn:rewrite_phi}
\varphi(x)=\Phi^\transp(\Phi\Phi^\transp+n\lambda I)^{-1}+n\lambda(\Phi^\transp\Phi+n\lambda I)^{-1}\varphi(y),
\end{equation}
which leads us to
\begin{equation}
\label{eqn:rewrite_phi_phit}
\begin{split}
\varphi(y)^\transp\varphi(x) &= K_y^\transp(\Phi\Phi^\transp+n\lambda I)^{-1}K_y \\
&+ n\lambda\varphi(y)^\transp(\Phi^\transp\Phi+n\lambda I)^{-1}\varphi(y).
\end{split}
\end{equation}
Finally, we have
\begin{equation}
\label{eqn:power_function}
\begin{split}
&n\lambda\varphi(y)^\transp(\Phi^\transp\Phi + n\lambda I)^{-1}\varphi(x)\\
=&k(y,y) - K_y^\transp(K+n\lambda I)^{-1}K_y.
\end{split}
\end{equation}

\subsubsection{Proof of Lemma~\ref{lem:est_uncertainty}}
\label{sec:proof_lem_1}

We have for all $c\in\mathcal{Y}$
\begin{align*}
&\abs{p_c(y)-\bar{p}_c(y)} = \abs{p_c(y)-K_y^\transp(K+n\lambda I)^{-1}\textbf{p}_\textbf{y}}\\
= &\abs{\varphi(y)^\transp p_c - \varphi(y)^\transp\Phi^\transp(\Phi\Phi^\transp+n\lambda I)^{-1}\Phi p_c} \tag*{Definitions in \secref{sec:derivations}}\\
= &\abs{\varphi(y)^\transp p_c- \varphi(y)^\transp(\Phi^\transp\Phi+n\lambda I)^{-1}\Phi^\transp\Phi p_c} \tag*{Eq.~\eqref{eqn:shift_phi}}\\
= &\abs{n\lambda\varphi(x)^\transp(\Phi^\transp\Phi+n\lambda I)^{-1}p_c}\tag*{Eq.~\eqref{eqn:rewrite_phi}}\\
\le &\norm{n\lambda\varphi(y)^\transp(\Phi^\transp\Phi + n\lambda I)^{-1}}_k\norm{p_c}_k\tag*{Cauchy-Schwartz}\\
\le &\sqrt{\Gamma}\norm{n\lambda\varphi(y)^\transp(\Phi^\transp\Phi + n\lambda I)^{-1}}_k\tag*{Assumption~\ref{ass:class_rkhs_bound}}\\
=&\sqrt{\Gamma}\sqrt{n\lambda\varphi(x)^\transp(\Phi^\transp\Phi+n\lambda I)^{-1}n\lambda(\Phi^\transp\Phi+n\lambda I)^{-1}\varphi(y)}\tag*{Definition in \secref{sec:derivations}}\\
\le &\sqrt{\Gamma}(n\lambda\varphi(x)^\transp(\Phi^\transp\Phi+n\lambda I)^{-1}(\Phi^\transp\Phi+n\lambda I)\\
&(\Phi^\transp\Phi+n\lambda I)^{-1}\varphi(y))^{\sfrac{1}{2}}\tag*{K is pos.\ def.\ }\\  
= &\sqrt{\Gamma}\sqrt{k(y,y)-K_y^\transp(K+n\lambda I)^{-1}K_y}\tag*{Eq.~\eqref{eqn:power_function}},
\end{align*}
from which the claim follows through Definition~\ref{def:power_function}.

\subsubsection{Proof of Lemma~\ref{lem:meas_uncertainty}}
\label{sec:proof_lem_2}

We have for all $c\in\mathcal{Y}$
\begin{align*}
&\abs{\bar{p}_c(y)-\check{p}_c(y)}=\abs{K_y^\transp(K+n\lambda I)^{-1}(\textbf{p}_\textbf{y}-\textbf{c})}\\
= &\abs{\varphi(y)^\transp\Phi(K+n\lambda I)^{-1}(\textbf{p}_\textbf{y}-\textbf{c})}\tag*{Definitions in \secref{sec:derivations}}\\
= &\abs{\varphi(y)^\transp(K+n\lambda I)^{-1}\Phi^\transp(\textbf{p}_\textbf{y}-\textbf{c})}\tag*{Eq.~\eqref{eqn:shift_phi}}\\
\le &\norm{\varphi(x)^\transp(K+n\lambda I)^{-\sfrac{1}{2}}}_k\norm{(K+n\lambda I)^{-\sfrac{1}{2}}\Phi^\transp(\textbf{p}_\textbf{y}-\textbf{c})}_k\tag*{Cauchy-Schwarz}\\
= &\sqrt{\varphi(x)^\transp(K+n\lambda I)^{-1}\varphi(y)}\\
&\sqrt{(\Phi^\transp(\textbf{p}_\textbf{y}-\textbf{c}))^\transp(K+n\lambda I)^{-1}\Phi^\transp(\textbf{p}_\textbf{y}-\textbf{c})} \\
=&\sqrt{\frac{1}{n\lambda}}\varrho(y)\sqrt{(\Phi^\transp(\textbf{p}_\textbf{y}-\textbf{c}))^\transp(K+n\lambda I)^{-1}\Phi^\transp(\textbf{p}_\textbf{y}-\textbf{c})}\tag*{Eq.~\eqref{eqn:power_function}, Def.~\ref{def:power_function}}\\
=&\sqrt{\frac{1}{n\lambda}}\varrho(y)\sqrt{(\textbf{p}_\textbf{y}-\textbf{c})^\transp K(K+n\lambda I)^{-1}(\textbf{p}_\textbf{y}-\textbf{c})}\tag*{Definitions in \secref{sec:derivations}}.
\end{align*}
The claim then follows from~\cite[Thm.~1]{fiedler2021practical} since the random variables are $\sigma$-sub Gaussian with $\sigma\le \frac{1}{4}$ following Lemma~\ref{lem:subGaussian}.

\subsection{Example for Definition~\ref{def:mmd}}
\label{sec:mmd_example}
Defining contexts based on $\eta$ might seem to be an abstract choice at first. 
Therefore, let us here make it more intuitive by calculating it for the Furuta pendulum example from \secref{sec:eval}.
When considering the balancing of the Furuta pendulum, the system can be approximated as a linear, time-invariant system.
If we further assume that the state measurements we receive are perturbed by Gaussian noise, the resulting trajectory data also follows a Gaussian distribution.
Given that we sub-sample the data such that it is approximately \iid, we can now analytically compute the MMD.
For ease of presentation, we consider a scalar state, \eg the angular velocity of the pole that we also used in \secref{sec:eval}.
Then, for two contexts $c_\mathrm{a}$ and $c_\mathrm{b}$ that generate Gaussian data distributions $\mathcal{N}(\mu_\mathrm{a},\sigma^2_\mathrm{a})$ and $\mathcal{N}(\mu_\mathrm{b},\sigma^2_\mathrm{b})$ with a Gaussian kernel $k_\mathrm{mmd}$ we have in the infinite sample limit
\begin{align*}
    \mmd(X_\mathrm{a}, X_\mathrm{b}) &= \frac{\exp(\frac{2\abs{\mu_\mathrm{a}}^2}{2(2\sigma_\mathrm{a}^2+\gamma^2)})}{\sqrt{2\pi(2\sigma_\mathrm{a}^2+\gamma^2)}} + \frac{\exp(\frac{2\abs{\mu_\mathrm{b}}^2}{2(2\sigma_\mathrm{b}^2+\gamma^2)})}{\sqrt{2\pi(2\sigma_\mathrm{b}^2+\gamma^2)}}\\ 
    &- \frac{2\exp(\frac{\abs{\mu_\mathrm{a}+\mu_\mathrm{b}}^2}{2(\sigma_\mathrm{a}^2+\sigma_\mathrm{b}^2+\gamma^2)})}{\sqrt{2\pi(\sigma_\mathrm{a}^2+\sigma_\mathrm{b}^2+\gamma^2)}},
\end{align*}
with $\gamma$ the length scale of the Gaussian kernel.
This result is based on the derivations from~\cite{rustamov2021closed}.

With the experimental data that we have, we can now approximate both mean and standard deviation for both contexts and compute the MMD.
We can also compute the corresponding $\eta$ from Proposition~\ref{prop:mmd}. 
When comparing both, we see that for differentiating the context ``no weight'' from ``weight 2'' and ``weight 1'' from ``weight 2,'' the 2500 data samples we collected are sufficient to have an $\eta$ that is below the threshold given in Proposition~\ref{prop:mmd}.
To be able to guarantee that we can differentiate context ``no weight'' from context ``weight 1,'' we would have needed around \num{350000} data points.
However, we see in the evaluation that also with 2500 data samples, we can reliably identify the context.

}

\bibliographystyle{IEEEtran}
\bibliography{IEEEabrv,ref}

\newpage

 

\begin{IEEEbiography}[{\includegraphics[width=1in,height=1.25in,clip,keepaspectratio]{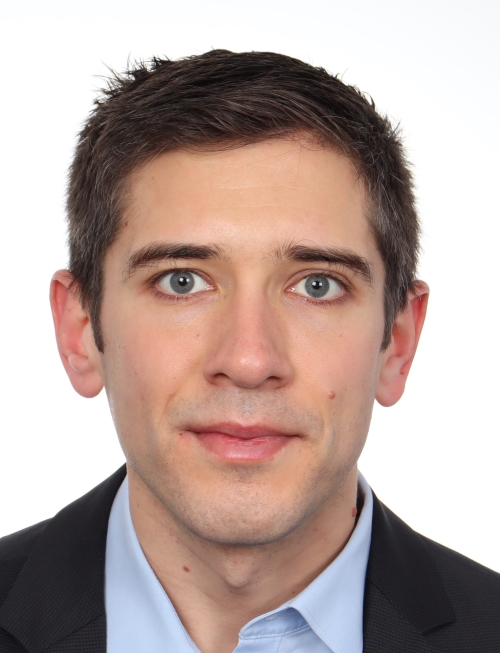}}]{Dominik Baumann}
is currently an assistant professor at Aalto University, Espoo, Finland.
Before, he received the Dipl.-Ing.\ degree in electrical engineering from TU Dresden, Germany.
He then was a joint Ph.D.\ student with the Max Planck Institute for Intelligent Systems in Germany and KTH Stockholm, Sweden, and received the Ph.D.\ degree, also in electrical engineering, in 2020.
After his Ph.D., he was a postdoctoral researcher at RWTH Aachen University, Germany, and Uppsala University, Sweden.
Dominik has received the best paper award at the 2019 ACM/IEEE International Conference on Cyber-Physical Systems, the best demo award at the 2019 ACM/IEEE International Conference on Information Processing in Sensor Systems, and the future award of the Ewald Marquardt Foundation.
His research interests revolve around learning and control for networked multi-agent systems.
More info at \url{https://baumanndominik.github.io/}.
\end{IEEEbiography}

\vspace{11pt}

\begin{IEEEbiography}[{\includegraphics[width=1in,height=1.25in,clip,keepaspectratio]{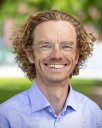}}]{Thomas Sch\"{o}n}
(Senior Member, IEEE) received
the B.Sc.\ degree in business administration and economics, the M.Sc.\ degree in applied physics and
electrical engineering, and the Ph.D.\ degree in automatic control from Linköping University, Linköping,
Sweden, in January 2001, September 2001, and
February 2006, respectively. He is currently the Beijer
Professor of artificial intelligence with the Department of Information Technology, Uppsala University,
Uppsala, Sweden. He has held visiting positions with
the University of Cambridge, Cambridge, U.K., the
University of Newcastle, Newcastle, NSW, Australia, and Universidad Técnica
Federico Santa María, Valparaíso, Chile. In 2018, he was elected to The Royal
Swedish Academy of Engineering Sciences (IVA) and The Royal Society of
Sciences at Uppsala. He was the recipient of the Tage Erlander prize for natural
sciences and technology in 2017 and the Arnberg prize in 2016, both awarded by
the Royal Swedish Academy of Sciences (KVA), Automatica Best Paper Prize
in 2014, and in 2013, Best Ph.D.\ Thesis Award by The European Association
for Signal Processing, and Best Teacher Award at the Institute of Technology,
Linköping University in 2009.

\end{IEEEbiography}

\vfill

\end{document}